\documentclass[%
reprint, 
 amsmath,amssymb,
 aps, physrev,
]{revtex4-2}

\usepackage{graphicx}% Include figure files
\usepackage{multirow}
\usepackage{dcolumn}% Align table columns on decimal point
\usepackage{bm}% bold math
\usepackage{float}
\usepackage{subfig}
\usepackage{comment}
\usepackage{amssymb}
\usepackage{amsthm}
\usepackage{amsmath}
\usepackage[colorlinks]{hyperref}% add hypertext capabilities
\hypersetup{linkcolor = blue, citecolor=green}
\newcommand\norm[1]{\left\lVert#1\right\rVert}

\newtheorem{theorem}{Theorem}

\usepackage{color}
\begin{document}

\preprint{APS/123-QED}

\title{Machine Learning Hamiltonian Dynamical Systems with Sparse and Noisy Data 
}% 

\author{Vedanta Thapar}
\email{vedanta.thapar@maths.ox.ac.uk}

\affiliation{The Mathematical Institute, University of Oxford, Oxford OX2 6GG, United Kingdom}

\author{Abhinav Gupta}%
\email{Contact author: abhinav.gupta@ststephens.edu}
\affiliation{Center for Theoretical Physics, St. Stephen's College, University of Delhi, Delhi 110007, India}%

\date{\today}% It is always \today, today,

\begin{abstract}
Machine learning has become a powerful tool for discovering governing laws of dynamical systems from data. However, most existing approaches degrade severely when observations are sparse, noisy, or irregularly sampled. In this work, we address the problem of learning symbolic representations of nonlinear Hamiltonian dynamical systems under extreme data scarcity by explicitly incorporating physical structure into the learning architecture. We introduce Adaptable Symplectic Recurrent Neural Networks (ASRNNs), a parameter-cognizant, structure-preserving model that combines Hamiltonian learning with symplectic recurrent integration, avoiding time derivative estimation, and enabling stable learning under noise. We demonstrate that ASRNNs can accurately predict long-term dynamics even when each training trajectory consists of only two irregularly spaced time points, possibly corrupted by correlated noise. Leveraging ASRNNs as structure-preserving data generators, we further enable symbolic discovery using independent regression methods (SINDy and PySR), recovering exact symbolic equations for polynomial systems and consistent polynomial approximations for non-polynomial Hamiltonians. Our results show that such architectures can provide a robust pathway to interpretable discovery of Hamiltonian dynamics from sparse and noisy data.

\end{abstract}

\maketitle

%\tableofcontents

\section{Introduction} \label{Intro}

Since the 1990s, the science of machine learning has garnered great interest in the natural sciences community and with its recent rapid advancements, it has emerged as a powerful tool for artificial scientific discovery with applications ranging from modelling multiscale physical systems to complex biological processes \cite{angelis2023artificial, makke2024interpretable, wang2023scientific, xu2021artificial}. At the largest scales, ML architectures such as AlphaFold \cite{alphafold} and GNoMe \cite{merchant2023scaling} are being used for scientific breakthroughs that have the potential of revolutionising the fields of drug discovery and materials science respectively. There have also been impressive advancements in studying complex dynamical phenomena using neural networks such as in modelling nuclear fusion processes \cite{arhouni2025artificial, pavone2023machine}, climate and weather science \cite{bracco2024machine, eyring2024pushing, lai2024machine, graphcast}, and the acceleration of fluid and molecular dynamics simulations \cite{brunton2024promising, kochkov2021machine, pravsnikar2024machine, vinuesa2022enhancing} among many others. Many of these approaches fall under the umbrella of physics-informed neural networks (PINNs), which explicitly incorporate prior physical knowledge into the learning process \cite{PINNsReview, toscano2024pinnspikansrecentadvances}. \\

A particularly important class of systems in this context is Hamiltonian dynamical systems, which exhibit rich behaviour including periodic, quasiperiodic, and chaotic dynamics \cite{goldstein2011classical}. Because their evolution is constrained by Hamilton’s equations, learning such systems while respecting their underlying structure has received considerable attention \cite{PINNsReview}. Hamiltonian Neural Networks (HNNs) \cite{greydanus2019hamiltonian} were amongst the first architectures to explicitly encode this structure and have been shown to outperform vanilla neural networks, especially in chaotic regimes \cite{orderandchaos}. Subsequent extensions include parameter-cognizant Adaptable HNNs \cite{Han_Glaz_Haile_Lai_2021}, symplectic recurrent architectures employing structure-preserving integrators (SRNNs) \cite{Chen2020Symplectic}, and further developments addressing inseparable Hamiltonians \cite{xiong2022nonseparablesymplecticneuralnetworks}, noncanonical coordinates \cite{cranmer2020lagrangianneuralnetworks, choudhary2021forecasting}, constraints \cite{contraints}, and dissipation \cite{sosanya2022dissipative}. Alternative structure-preserving approaches, such as SympNets \cite{SympNets}, aim to learn symplectic maps directly, and recent work has unified many of these methods under the framework of Generalised Hamiltonian Neural Networks (GHNNs) \cite{HORN2025113536}. \\

Despite their predictive success, most Hamiltonian learning architectures remain largely black-box models, providing blurred insight into the underlying governing laws. This has motivated growing interest in interpretable machine learning approaches, particularly symbolic regression methods aimed at recovering explicit equations of motion from data \cite{allen2023interpretable, makke2024interpretable}. Prominent examples employed in this work include SINDy \cite{brunton2016discovering} which targets sparse representations of nonlinear ODEs via regression over a fixed function library, and the more general PySR framework \cite{cranmer2023interpretable}, an evolutionary symbolic regression method that does not require a predefined basis.
Learning equations for Hamiltonian systems using physical priors has also been explored previously, for example in \cite{SSINs}. However, in the interest of interpretability, that approach avoids deep neural networks and instead combines SINDy with symplectic integration, limiting expressibility to the chosen basis and precluding parameter adaptability. Other symbolic regression approaches include probabilistic and Bayesian equation discovery methods \cite{StochasticComplexSYstems, BayesianScientist}, as well as physics-specific ``artificial scientist'' models such as SciNet, AIFeynman, AIPoincare, and AINewton \cite{SciNet, AIFeynman, AIPoincare, AINewton}. Importantly, these methods typically degrade severely when data are sparse, irregularly sampled, or noisy---conditions common in realistic experimental and observational settings \cite{hsin2024symbolicregressionsparsenoisy}. \\

In this work we address the following question: given a parametrised dynamical system and access to only very sparse, irregular, and noisy trajectory data for a limited set of parameter values, can one recover symbolic representations of the underlying Hamiltonian dynamics by incorporating appropriate physical priors? We propose Adaptable Symplectic Recurrent Neural Networks (ASRNNs), which combine parameter-cognizant Hamiltonian learning with symplectic recurrent integration focused on separable Hamiltonian systems. This allows the generation of regularly spaced trajectories for unseen parameter values, providing suitable data for symbolic discovery via PySINDy \cite{brunton2016discovering} and PySR \cite{cranmer2023interpretable}. This provides a modular approach that improves downstream interpretability without reducing expressibility. \\
We evaluate ASRNNs constructed using small multilayer perceptrons (MLPs) on two representative systems: a two-parameter Henon–Heiles system and the single parameter Morse potential. While both systems have been considered previously in the context of Adaptable HNNs \cite{Han_Glaz_Haile_Lai_2021}, prior approaches assumed regularly sampled, noise-free data and struggled in multi-parameter settings. In contrast, ASRNNs trained on highly sparse and irregular observations generalise well across parameter space and, when used with symbolic regression, recover equations of motion and Hamiltonian structure despite severe data limitations. We also consider the ability of these architectures to extrapolate to parameter regimes with qualitatively different dynamics not included in the training set. \\

The rest of the paper is organised as follows. The following section introduces the general structure of ASRNNs, Section \ref{HH} introduces the physical systems studied, Section \ref{TD} describes data generation and training procedures, Section \ref{results} presents empirical results, Section \ref{theory} constructs a theoretical analysis of the ASRNN loss gradient under noise and Section \ref{conclusion} concludes with a discussion of limitations and future directions.

\section{Model structure} \label{ASRNN structure}

A Hamiltonian Neural Network (HNN) \cite{greydanus2019hamiltonian} parametrises a time-independent Hamiltonian $\mathcal{H} = \mathcal{H}(\mathbf{q}, \mathbf{p})$ using a neural network $\mathcal{H}_{\theta} (\mathbf{q}, \mathbf{p})$ where $(\mathbf{q}, \mathbf{p})$ are canonical coordinates. Hamilton’s equations are enforced during training via the loss,
\begin{equation}
    \mathcal{L}_{(\mathbf{q}, \mathbf{p})} (\theta) = \norm{\frac{\partial \mathcal{H}_{\theta}}{\partial \mathbf{p}} - \frac{\partial \mathbf{q}}{\partial t}}_2 + \norm{\frac{\partial \mathcal{H}_{\theta}}{\partial \mathbf{q}} + \frac{\partial \mathbf{p}}{\partial t}}_2,
    \label{HNNloss}
\end{equation}
which requires access to time derivatives $(\dot{\mathbf{q}}, \dot{\mathbf{p}})$.\\
Adaptable HNNs extend this framework by introducing parameter channels $\lambda$, allowing Hamiltonians of the form $\mathcal{H}(\mathbf{q}, \mathbf{p}; \mathbf{\lambda})$ to be learned across parameter space \cite{Han_Glaz_Haile_Lai_2021}. However, both HNNs and AHNNs rely on derivative estimates (such as by finite differences), which are typically unreliable when data are sparse and noisy. To address this, Symplectic Recurrent Neural Networks (SRNNs) 
\cite{Chen2020Symplectic} replace derivative-based losses with recurrent integration using symplectic integrators, but are restricted to separable Hamiltonians $\mathcal{H}(\mathbf{q}, \mathbf{p}) = \mathcal{K}(\mathbf{p}) + \mathcal{V}(\mathbf{q})$ where $\mathcal{K}, \mathcal{V}$ are kinetic and potential energy functions respectively. \\
We propose ASRNNs (Adaptable Symplectic Recurrent Neural Networks) which combine parameter-cognizant Hamiltonian learning with symplectic recurrence. ASRNNs model Hamiltonians of the form
\begin{equation}
    \mathcal{H}(\mathbf{q}, \mathbf{p} ; \lambda) = \mathcal{K}(\mathbf{p}) + \mathcal{V}(\mathbf{q}; \lambda),
\end{equation}
where coordinates are scaled such that parameter dependence enters only through the potential. The model thus consists of two neural networks, i.e. $\mathcal{H}_{\theta}(\mathbf{q}, \mathbf{p}; \lambda) = \mathcal{K}_{\theta_1}(\mathbf{p}) + \mathcal{V}_{\theta_2}(\mathbf{q}; \lambda)$ where $\theta = (\theta_1, \theta_2)$. A Verlet integrator is used to perform the recurrent steps,
\begin{gather}
    \mathbf{p}\left(t+\frac{\Delta t}{2}\right) = \mathbf{p}(t) - \frac{\Delta t}{2} \frac{\partial \mathcal{V}_{\theta_2}}{\partial \mathbf{q}} \bigg|_t \\
    \mathbf{q}\left( t+\Delta t\right) = \mathbf{q}(t) + \Delta t \frac{\partial \mathcal{K}_{\theta_1}}{\partial \mathbf{p}} \bigg|_{t+\frac{\Delta t}{2}} \\
    \mathbf{p}\left( t+\Delta t\right) =  \mathbf{p}\left( t+\frac{\Delta t}{2}\right) - \frac{\Delta t}{2} \frac{\partial \mathcal{V}_{\theta_2}}{\partial \mathbf{q}} \bigg|_{t+\Delta t},
\end{gather}
where $\Delta t$ is a chosen time step. Training minimizes the trajectory-matching loss,
\begin{equation}
    \mathcal{L}_{(\mathbf{q}_0, \mathbf{p_0}; \lambda)} (\theta) = \sum_{t \in \mathcal{T}} \left(\norm{\mathbf{q}(t) - \mathbf{\hat{q}(t)}}_2 + \norm{\mathbf{p}(t) - \mathbf{\hat{p}(t)}}_2  \right),
    \label{ASRNNloss}
\end{equation}
where $(\mathbf{q}_0, \mathbf{p_0}; \lambda)$ are initial conditions, $(\mathbf{q}(t), \mathbf{p}(t))$ and $(\mathbf{\hat{q}(t)}, \mathbf{\hat{p}(t)})$ are the ASRNN outputs and ground truth values for time $t$ respectively, and $\mathcal{T}$ is the set of observed time points. The ASRNN recurrent unit is illustrated in Fig.~\ref{fig:ASRNN structure}.
\begin{figure} 
    \centering
   \includegraphics[width=0.8\linewidth]{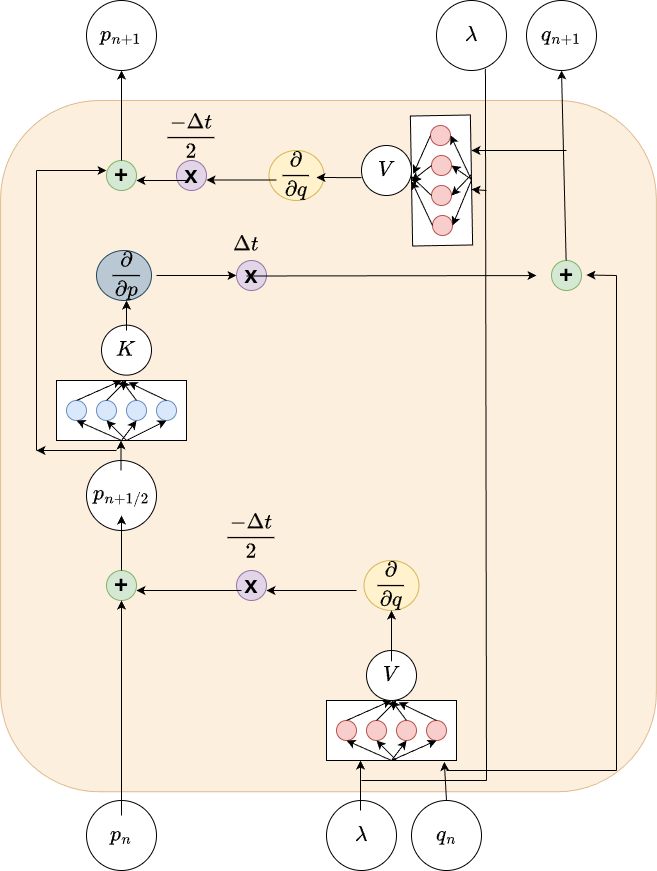}
    \caption{Recurrent unit of an ASRNN, here we take $(\mathbf{q_n}, \mathbf{p_n}; \lambda)$ as input and implement the Verlet algorithm which returns $(\mathbf{q_{n+1}}, \mathbf{p_{n+1}}; \lambda)$ where the $n$ indexes time steps separated by the chosen time step $\Delta t$.}
    \label{fig:ASRNN structure}
\end{figure}
All neural networks are implemented as small MLPs with two to three hidden layers and 30–50 neurons per layer. We employ random Gaussian weight initialisation; empirical results indicate that while initialisation does not significantly impact the performance within the training parameter regime, it can affect generalisation outside the training parameter regime---for a discussion see Appendix \ref{appendix1}.

\section{Illustrative systems} \label{HH}

To evaluate the ability of ASRNNs to learn Hamiltonian dynamics under extreme sparsity and noise, we follow previous work \cite{Han_Glaz_Haile_Lai_2021} and consider two representative systems: a polynomial, multi-parameter Hamiltonian (the Henon–Heiles system) and a non-polynomial, single-parameter system (the Morse potential).

\subsection{Henon-Heiles Hamiltonian}

The Henon-Heiles Hamiltonian is a low-dimensional system exhibiting a wide variety of dynamical behaviour including chaotic, quasiperiodic, resonant behaviour, bifurcations and periodic motion \cite{goldstein2011classical}. It is constructed in 2-dimensional configuration space (4-dimensional phase space) with two perpendicular linear harmonic oscillators coupled through an asymmetric nonlinear perturbation, it is given as
   \begin{equation}
    \mathcal{H}(\mathbf{q}, \mathbf{p}; \lambda) = \frac{p_1^2 + p_2^2}{2} + \frac{q_1^2 + q_2^2}{2} + \alpha q_1^2 q_2 - \beta \frac{q_2^3}{3},
    \label{Hamiltonian}
\end{equation} 
with equations of motion
\begin{align}
    \dot{q}_1 &= p_1, \label{eom1} \\
    \dot{q}_2 &= p_2, \\
    \dot{p}_1 &= -q_1 - 2\alpha q_1q_2, \\
    \dot{p}_2 &= -q_2 - \alpha q_1^2 + \beta q_2^2, \label{eom4}
\end{align}
where $\mathbf{q} = (q_1 \ q_2)^T$, $\mathbf{p} = (p_1 \ p_2)^T$ and $\lambda = (\alpha, \beta)$. Although often studied in the single parameter regime $\alpha=\beta$, we consider the full two parameter formulation. It is well established that for $\beta = \alpha \in (0, 1]$, the potential has a finite escape energy : orbits for $\mathcal{H}\leq\frac{1}{6}$ remain bounded for all initial conditions. 
Increasing either the energy or the parameter values leads to progressive destruction of invariant tori and the onset of chaos---for high parameters and energies close to the escape energies most orbits are chaotic.

Thus the system displays enough diversity, complexity and instability in its dynamics to serve as a stringent example for evaluating robustness and generalisation. We train ASRNNs on a sparse subset of parameter values with  $\alpha, \beta \in [0.2, 0.8]$ and assess performance across the full square $[0.2, 0.8]^2 \subset \mathbb{R}^2$.

\subsection{Morse potential}

The Hamiltonian for this system is given as $\mathcal{H} = \frac{p^2}{2} + \mathcal{V}(q)$ where
\begin{equation}
    \mathcal{V}(q; \alpha) = (1-e^{-\alpha (q-1)})^2 -1
\end{equation}
This potential is chosen as an example to study a system with a non-polynomial potential function, further, it is also interesting as it exhibits strongly asymmetric dynamics: on one side of the potential, trajectories experience an abrupt momentum reversal akin to a `rigid bounce' due to the steep exponential rise (which effectively acts as a hard wall). On the other side the force decays rapidly to 0. The Morse potential therefore provides a challenging test for learning sharp, non-polynomial features from sparse and noisy data. \\
It is hence interesting to study if the ASRNN, trained on very sparse (and possibly noisy) data, would capture this behaviour for unseen values of the parameter $\alpha$.

\section{Training data generation} \label{TD}

To evaluate robustness to noise in detail, we corrupt training data using temporally correlated Ornstein–Uhlenbeck (OU) \cite{OUprocess} noise, which provides a more realistic model of measurement uncertainty than uncorrelated Gaussian noise.

\subsection{Trajectory generation}

Ground truth trajectories are generated by integrating Hamilton’s equations using a symplectic leapfrog (Störmer–Verlet) integrator. Initial conditions $(p_0, q_0)$ are sampled such that $H(p_0, q_0) \leq E_{\max}$ (to ensure bounded orbits). Trajectories are first generated at fine temporal resolution and then coarsened to an observation timestep $\Delta t=0.1$.

\subsection{Ornstein-Uhlenbeck noise model}

The OU process \cite{OUprocess} is defined by
\begin{equation}
d\eta = -\phi \eta dt + \sigma dW,
\end{equation}
where $\phi > 0$ is the mean-reversion rate, $\sigma$ controls noise intensity, and $W$ denotes a standard Wiener process. The stationary variance of this process is $\sigma^2/(2\phi)$, and the autocorrelation decays exponentially with characteristic time $\tau = 1/\phi$. We parametrise noise using the correlation time $\tau$ and a noise-to-signal ratio (NSR), defined as the ratio of the noise standard deviation to the pooled within-trajectory standard deviation of the clean $(q, p)$ coordinates (which sets a natural scale). For each trajectory, we compute the variance about that trajectory's temporal mean, then average across all trajectories to obtain the signal variance against which the noise magnitude is scaled. We consider NSR values of $5\%$ and $10\%$ and correlation times $\tau \in \{\Delta t/5, \Delta t, 5 \Delta t, 25\Delta t\}$ spanning regimes from effectively i.i.d. Gaussian noise to fluctuations correlated across the full sampling window of 15 steps.\\
For discrete observations at intervals $\Delta t$, the OU process admits an exact update
\begin{equation}
    \eta_{n+1} = a \, \eta_n + \sqrt{1 - a^2} \, \sigma_\infty \, \xi_n,
\end{equation}
where $a = \exp(-\Delta t / \tau)$, $\sigma_\infty = \sigma / \sqrt{2\phi}$, and $\xi_n\sim \mathcal{N}(0, 1)$. The noisy observations are then $(\tilde{p}_n, \tilde{q}_n) = (p_n, q_n) + \eta_n$.

\subsection{Sparse sampling}

In realistic experimental conditions where continuous monitoring is impractical, observations are often irregular and sparse. To simulate this we construct sliding training windows spanning 15 time instants from the generated data, and take $N=2$ measurements from each window: the initial time $t=0$ which serves as the integration starting point, and a second observation drawn uniformly from the subsequent time steps, with a maximum separation of 14 steps. This setup represents an extreme sparsity regime, reducing each trajectory segment to an initial condition and a single future observation. 
The chosen OU correlation times span noise processes that range from short to long-correlated relative to this sampling window. For a given NSR, a well-performing model should be robust to noise at smaller values of $\tau$.

\section{Results} \label{results}

\subsection{Henon-Heiles system}

ASRNNs were implemented using MLPs for the kinetic and potential energy components. Best performance was obtained using three hidden layers with 20–50 neurons per layer; we fix 30 neurons per layer for all reported results. These networks are substantially smaller than those used in previous work (e.g. \cite{Han_Glaz_Haile_Lai_2021} employed layers of 200 neurons). The $tanh$ activation function is used throughout to ensure smoothness. The resulting architectures are $\{2, 30, 30, 30, 1\}$ for $\mathcal{K}_{\theta_1}$ and $\{4,  30, 30, 30, 1\}$ for $\mathcal{V}_{\theta_2}$ \footnote{More complex architectures were tested but did not improve performance under extreme sparsity.}.\\

Training data were generated as described in Section~\ref{TD}. For the Henon–Heiles system, we restrict training to parameter values $(\alpha, \beta) \in \{0.2, 0.4, 0.6, 0.8\}^2$ yielding $16$ parameter pairs. For each pair, $800$ sparse trajectory segments were constructed, each consisting of an initial condition and a single future observation separated by at most $14$ time steps, resulting in $9600$ training samples and $3200$ validation samples. Kinetic and potential energy networks were trained jointly in a recurrent setting by minimizing Eq.~\eqref{ASRNNloss} using the LBFGS optimiser \cite{LBFGS} for 500 epochs. We trained an ensemble of 40 models with independent random initialisations of both data slices and network weights. Typical training times were a few minutes per model on an Nvidia 4070Ti Super GPU with 16GB VRAM.

\subsubsection{Prediction} \label{prediction}

After training, ASRNNs were evaluated on trajectories generated at unseen parameter values and compared against ground truth Verlet integrations.
To ascertain the quality of our predictions we considered both the generated trajectories themselves as well as the kinetic and potential energies predicted by the networks, i.e. $\mathcal{K}_{\theta_1}$ and $\mathcal{V}_{\theta_2}$ respectively. For clarity of notation we define the following quantities to study the quality of predictions:
\begin{itemize}
    \item $\mathcal{H}_{\theta}$ : the Hamiltonian parametrised by the model, i.e. the raw output of the neural network
    \item $\mathcal{H}_{\text{pred}}$ : the true Hamiltonian (for example Eq. \ref{Hamiltonian}) evaluated on the predicted trajectories.
    \item Thus to quantify the accuracy of our predictions we can measure the difference in the energies of the predicted and ground truth trajectories, i.e. define 
\begin{equation}
    \epsilon := \left|\frac{\mathcal{H} - \mathcal{H}_{\text{pred}}}{\mathcal{H}}\right|,
    \label{DeltaH}
\end{equation}
as the fractional error.
\end{itemize}

Despite being trained on only two observations per trajectory window, ASRNNs accurately predict trajectories for several hundred time steps, including at unseen parameter values. Representative examples for
unseen parameters $(0.5, 0.7)$ are shown in Figs.~\ref{fig:57_example} and \ref{fig:57_example_withnoise} for noise-free and noisy training respectively. For a small subset of initial conditions prediction errors grow at long times, particularly outside the training parameter region consistent with the extreme sparsity of the data. However, for typical trajectories, prediction trajectories are remarkably close as measured by $\mathcal{H}_{\theta}$ which is conserved to high accuracy, with typical fluctuations of order $10^{-4}$ even under significant noise. As expected $\mathcal{H}_{\theta}$ differs from the true energy by a constant offset, which does not affect the dynamics. Fig.~\ref{fig:energyerror} summarises mean percentage fractional energy errors ($\epsilon \times 100 \%$) across parameter space: approximately $0.19\%$ (in the case without noise) increasing to $0.58\%$ and $2.6\%$ (in the case with $10\%$ NSR noise and $\tau = \Delta t/5, 25\Delta t$ respectively). Even in these regimes, errors remain substantially lower than those reported for AHNNs \cite{Han_Glaz_Haile_Lai_2021}, despite the use of much sparser, irregular and noisy data. Notably the performance is significantly weaker outside the training region $[0.2, 0.8]^2$, for a discussion on the ability of ASRNNs to extrapolate to qualitatively different dynamical regimes see Appendix \ref{appendix1}. Fig.~\ref{fig:energy_noise_HH} shows median $\epsilon$ as a function of noise for both seen and unseen parameter values. While noise increases error, the effect remains modest, apart from rare larger deviations arising primarily from outlier initial conditions.

\begin{figure} 
    \centering
    \includegraphics[width=\linewidth]{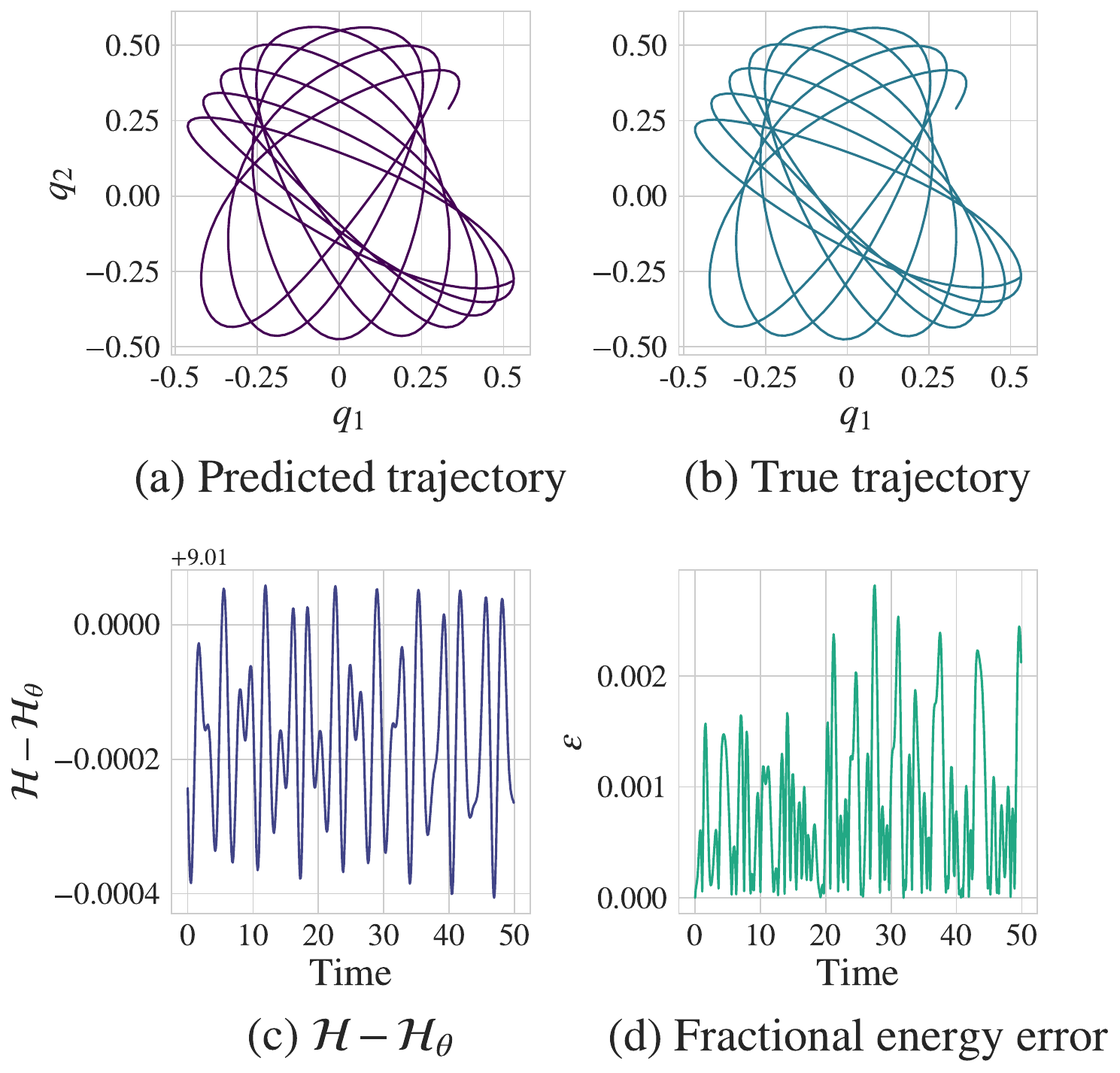}\\
    \caption{ASRNN predictions and error for a pair of unseen parameters $(0.5, 0.7)$ trained without noise. This is an example trajectory of $500$ time steps at energy $\approx 0.16$.}
    \label{fig:57_example}
\end{figure}

\begin{figure} 
    \centering
    \includegraphics[width=\linewidth]{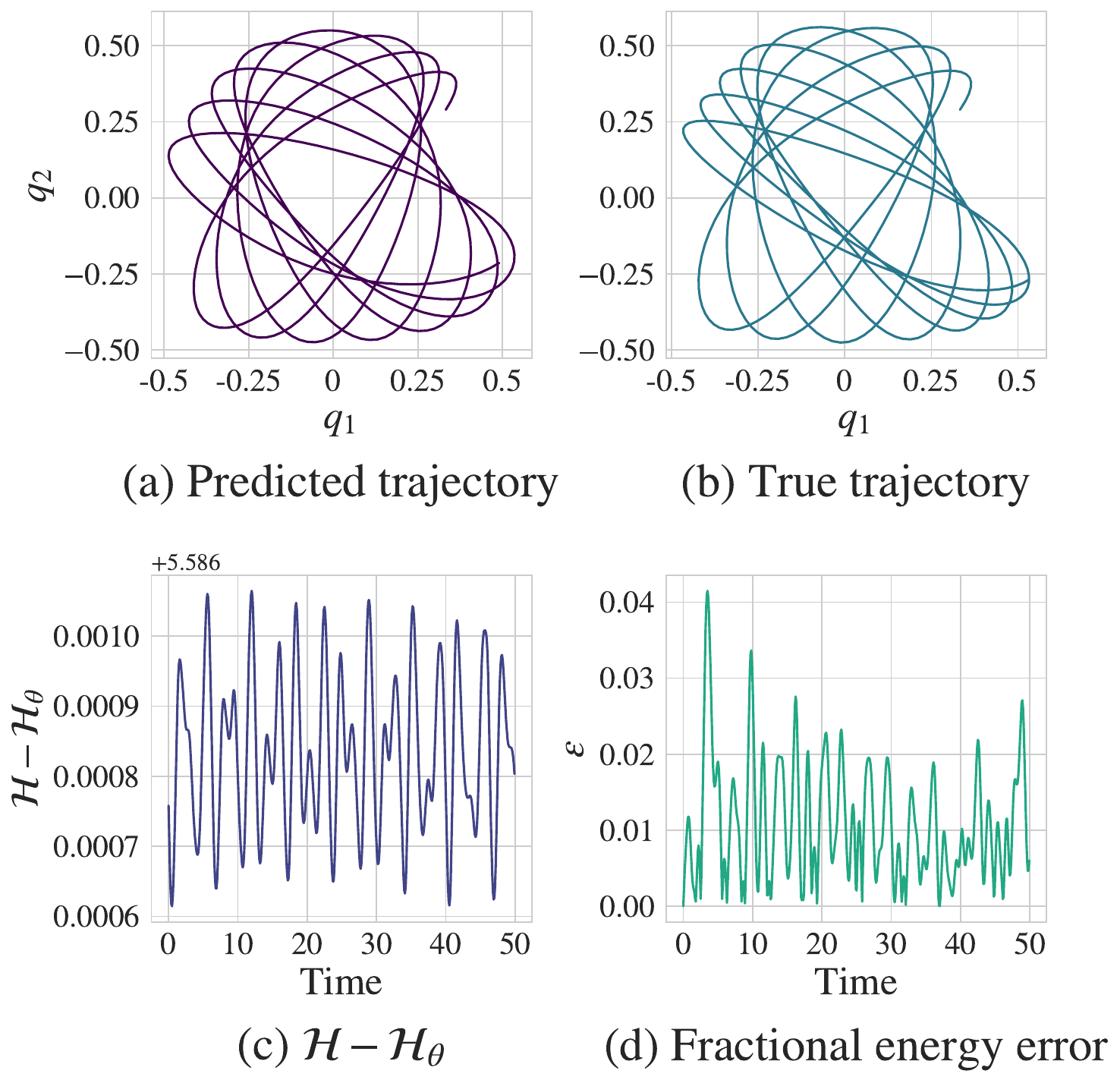}
    \caption{ASRNN predictions and error for a pair of unseen parameters $(0.5, 0.7)$ trained with NSR 10\% and $\tau=25\Delta t$, i.e. highly correlated noise. This is an example trajectory of $500$ time steps at energy $\approx 0.16$.}
    \label{fig:57_example_withnoise}
\end{figure}

\begin{figure*}
    \centering
    \subfloat[No noise]{\includegraphics[width=0.33\linewidth]{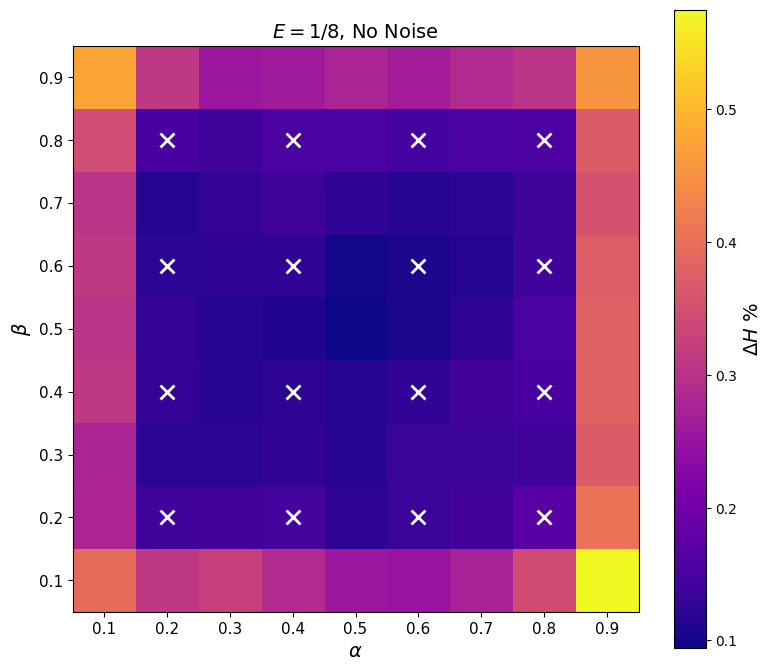}}
    \subfloat[NSR 10\%, $\tau=\Delta t/5$]{\includegraphics[width=0.33\linewidth]{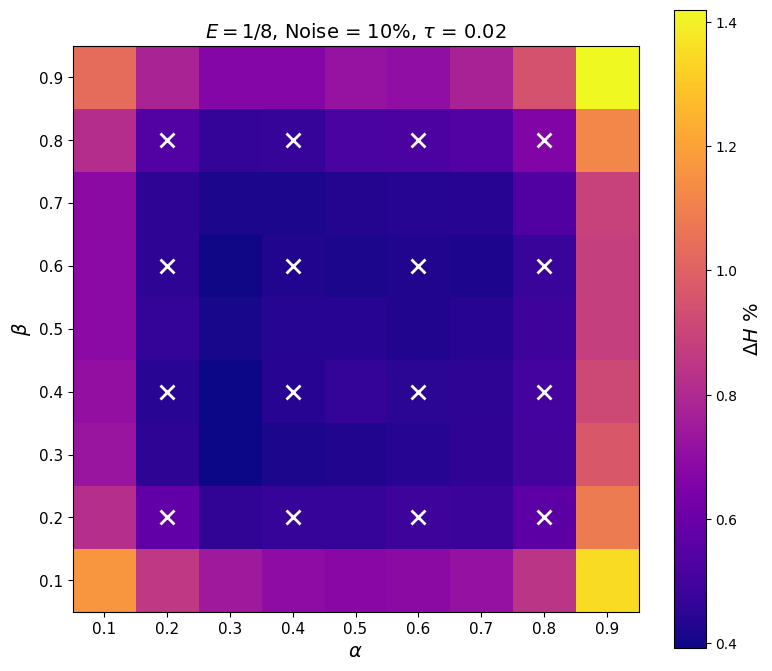}}
    \subfloat[NSR 10\%, $\tau=25 \Delta t$]{\includegraphics[width=0.33\linewidth]{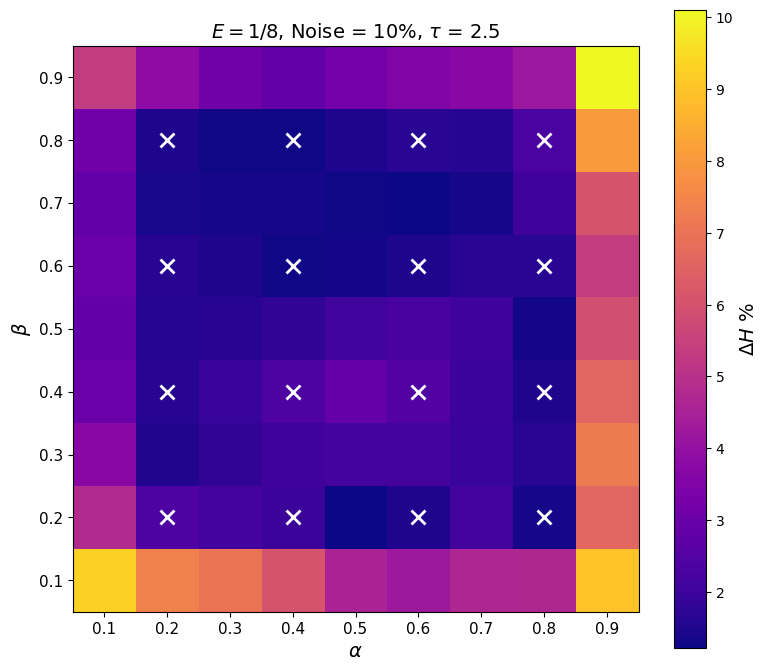}}
    \caption{Percentage fractional error (i.e. $\epsilon\times 100$) calculated by averaging over 40 trajectories of 500 steps at energy 1/8 for 81 parameter pairs in $[0.1, 0.9]^2$ (the white crosses indicate parameter points in the training set). This is calculated using trajectories generated by ASRNNs for three cases, one without noise and the others with significant noise and different extents of correlation, note the difference on the colorbar limits in the different cases.}
    \label{fig:energyerror}
\end{figure*}

\begin{figure} 
    \centering
    \includegraphics[width=\linewidth]{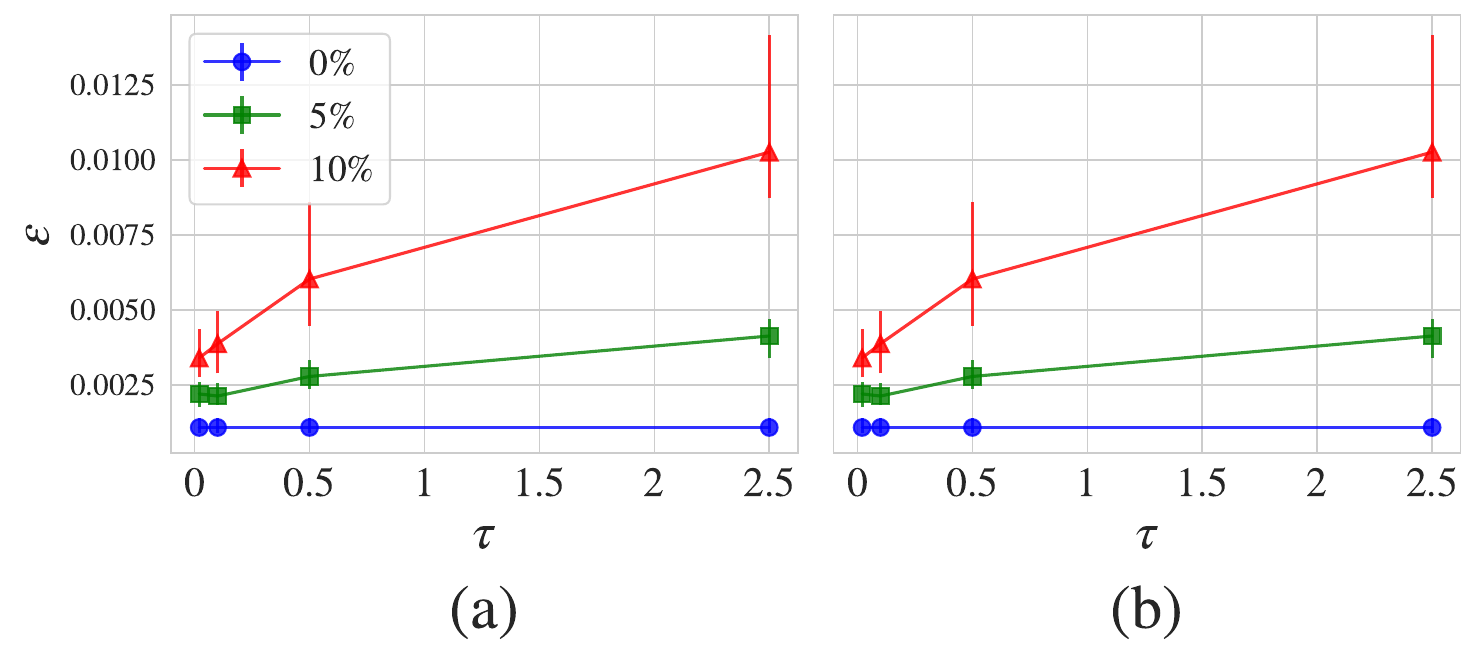}
    \caption{Median fractional energy error from ASRNN predictions computed over different extents of noise as a fixed parameter set for the Henon Heiles potential. Panel (a) indicates the results from a seen parameter set (0.4, 0.6) and panel (b) considers an unseen set (0.5, 0.7). The median is computed over 40 models each with 40 sample initial conditions and errorbars indicate the $25$th and $75$th percentiles.}
    \label{fig:energy_noise_HH}
\end{figure}

\subsubsection{Symbolic Regression}

\begin{figure} 
    \centering
    \includegraphics[width=\linewidth]{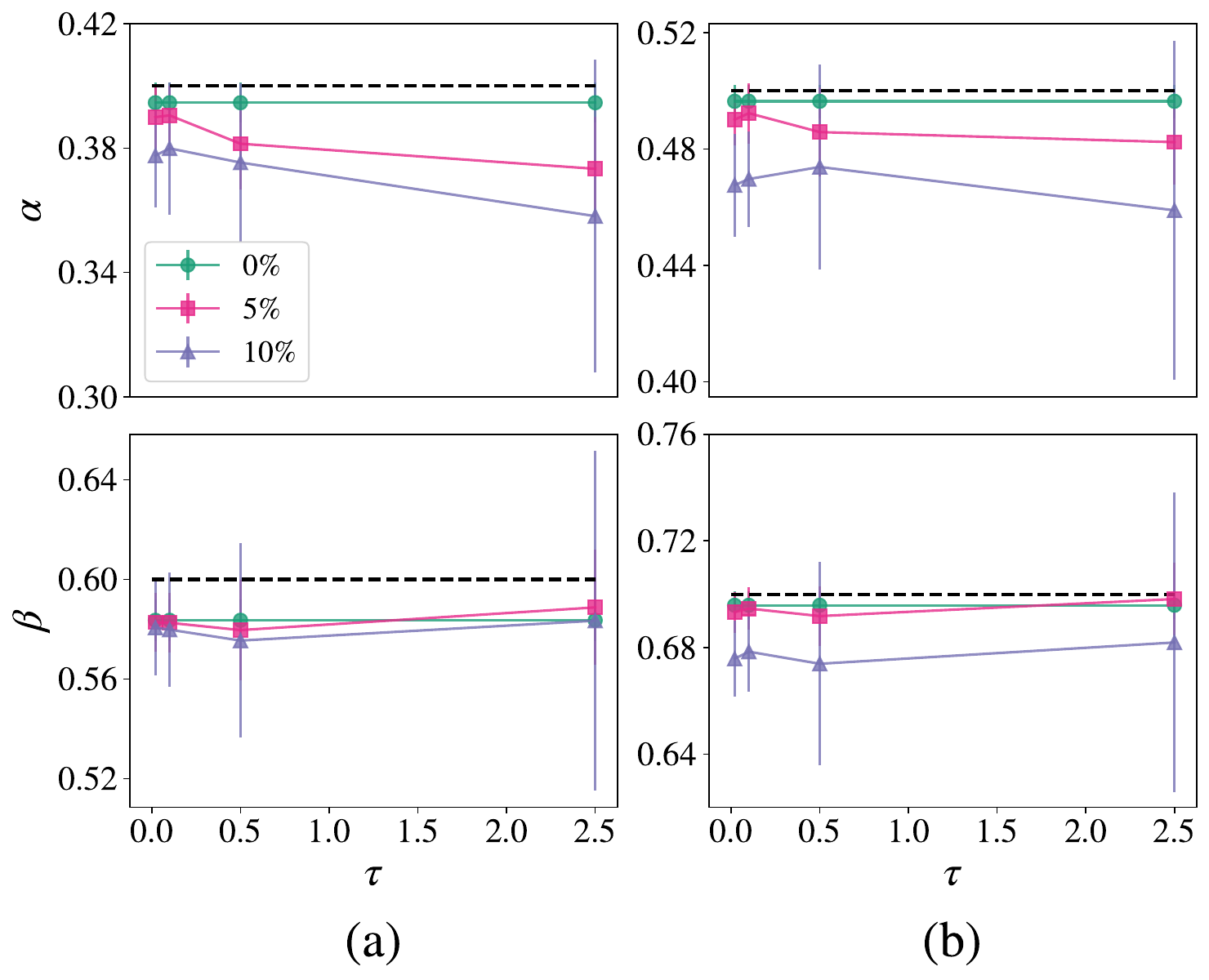}
    \caption{Estimated parameters obtained from symbolic regression with PySINDy for data generated by trained ASRNNs (with various extents of noise) with the Henon Heiles potential. Panel (a) indicates the results from a seen parameter set (0.4, 0.6) and panel (b) considers an unseen set (0.5, 0.7). Data is generated by an ensemble of 40 models and regressed independently, the estimated parameters are then averaged over this ensemble, errorbars indicate $1\sigma$ deviations.}
    \label{fig:HH_SR}
\end{figure}

To enable interpretability, ASRNNs are used as data generators for symbolic regression. Long, regularly spaced trajectories are generated and passed to PySINDy \cite{brunton2016discovering} and PySR \cite{cranmer2023interpretable}. The former is specifically designed to identify nonlinear dynamical systems of the form $\mathbf{\dot{x}} = f(\mathbf{x})$ while PySR is more flexible and can learn symbolic forms of general functions $f(\mathbf{x})$. PySINDy takes as input trajectories of $(\mathbf{q}, \mathbf{p})$ along with $\left(\frac{\partial \mathcal{K}_{\theta}}{\partial \mathbf{p}}, -\frac{\partial\mathcal{V}_{\theta}}{\partial \mathbf{q}}\right)$ and reliably recovers the exact equations of motion (Eqs.~\ref{eom1}–\ref{eom4}), including accurate estimates of $\alpha, \beta$ even when trained on data generated by noisy ASRNNs---see Fig.~\ref{fig:HH_SR}. PySR (defined with polynomial functions only), applied independently to kinetic and potential energies, recovers the correct functional forms up to constant offsets, yielding the full Hamiltonian (Eq.~\ref{Hamiltonian}). As noise magnitude and correlation increase, parameter variance grows, and due to the stochastic nature of PySR, small spurious terms occasionally appear on multiple independent runs. However, their coefficients are typically orders of magnitude smaller and do not affect the recovered dynamics. Full symbolic expressions are reported in the Supplementary Material.

\subsection{Non-polynomial systems - Morse potential} \label{Morse}

For the Morse potential, we trained ASRNNs with kinetic and potential energy networks with two hidden layers of 50 neurons each. Training followed the same sparse sampling protocol as before---for each parameter value $\alpha \in [0.5, 1.0, 2.0, 4.0]$, $800$ trajectory segments were constructed, each containing two observations separated by at most 14 time steps. \\

In the absence of noise, ASRNNs accurately reproduce the system dynamics and generalise well to unseen parameter values. Fig.~\ref{fig:morse_potential} shows the learned potential $\mathcal{V}_{\theta_2}(q)$ for various $\alpha$ under different extents of noise. Under noisy training, performance decays more strongly than in the Henon–Heiles case (Fig.~\ref{fig:Morse_energy_error}), especially for unseen parameters. This is primarily due to the struggle in capturing the sharp rebound behaviour in the presence of strongly correlated noise. As illustrated in Fig.~\ref{fig:morse_example2}, predicted and true trajectories under noise remain close at short times but diverge over longer horizons.

\begin{figure*}
    \centering
    \subfloat[No noise]{\includegraphics[width=0.25\linewidth]{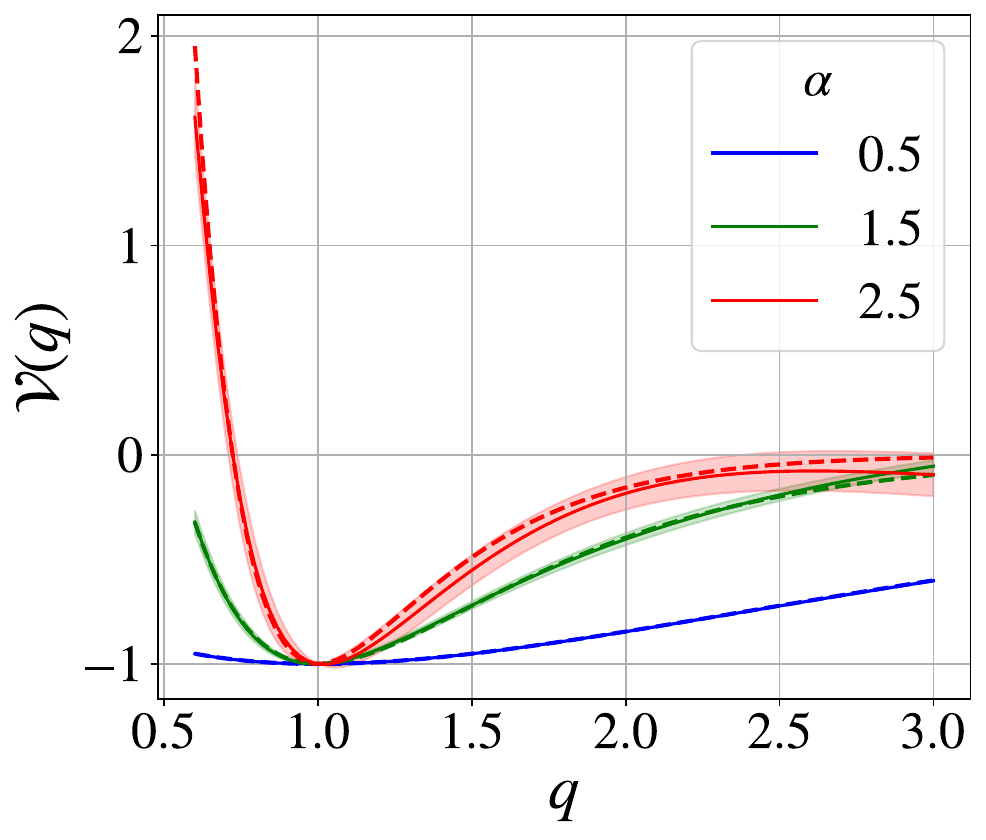}}
    \subfloat[NSR 5\%, $\tau=\Delta t/5$]{\includegraphics[width=0.25\linewidth]{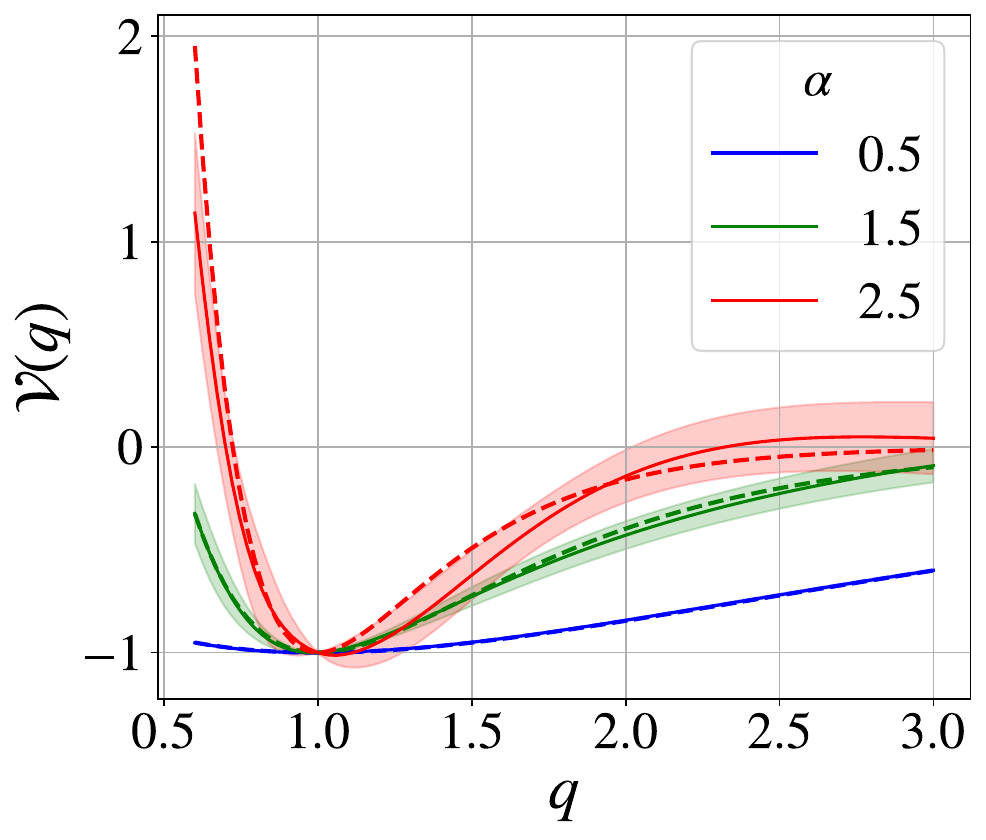}}
    \subfloat[NSR 5\%, $\tau=25\Delta t$]{\includegraphics[width=0.25\linewidth]{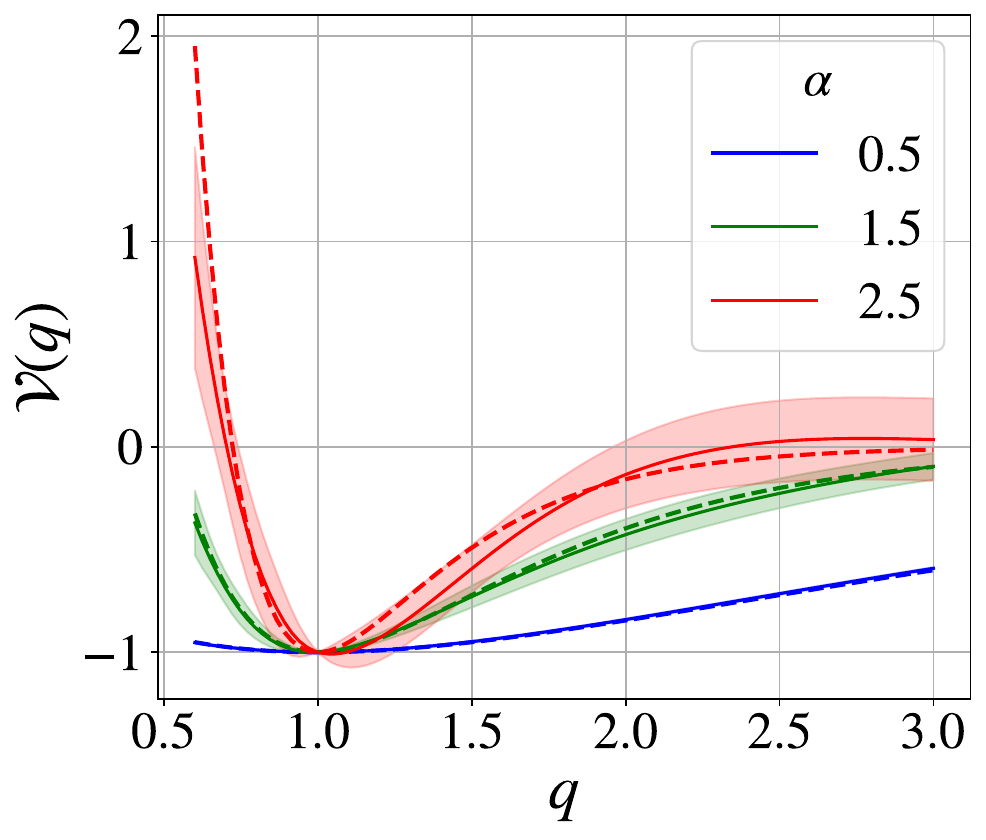}}
    \subfloat[NSR 10\%, $\tau=25\Delta t$]{\includegraphics[width=0.25\linewidth]{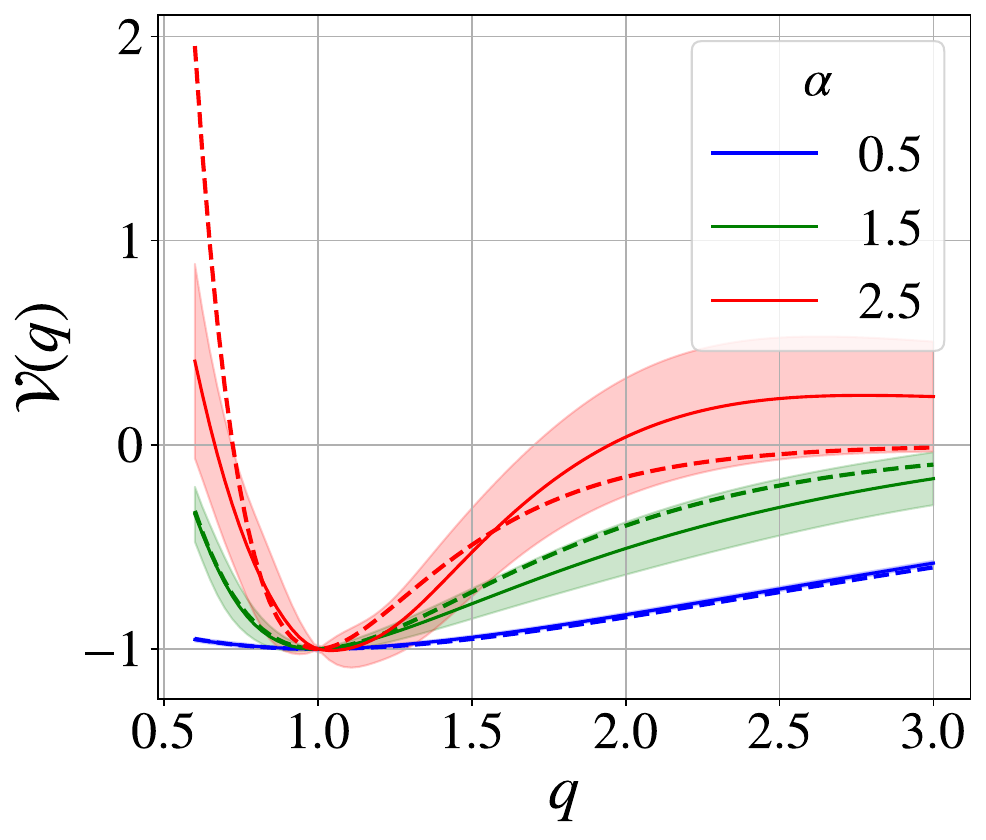}}
    \caption{Learned potential energy functions for the Morse system under different extents of noise averaged over an ensemble of 40 models, the regions indicate $1\sigma$ deviation and dashed line indicates the true potential curve. Here $\alpha=0.5$ is a seen parameter while the others are unseen.}
    \label{fig:morse_potential}
\end{figure*}

\begin{figure} 
    \centering
    \includegraphics[width=\linewidth]{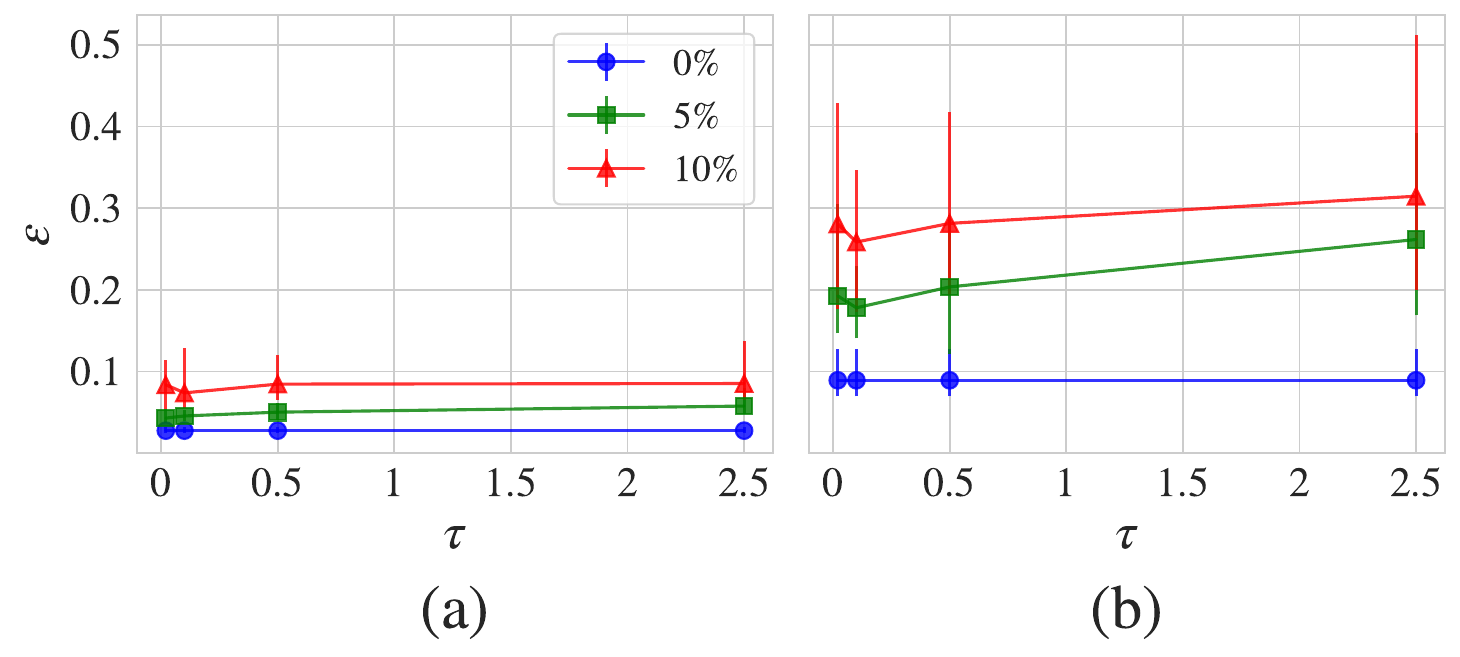}
    \caption{Median fractional energy error from ASRNN predictions computed over different extents of noise as a fixed parameter set for the Morse potential. Panel (a) indicates the results from a seen parameter $\alpha=2$ and panel (b) considers an unseen $\alpha=1.5$. The median is computed over 40 models each with 40 sample initial conditions and errorbars indicate the $25$th and $75$th percentiles.}
    \label{fig:Morse_energy_error}
\end{figure}

\begin{figure} 
    \centering
    \includegraphics[width=\linewidth]{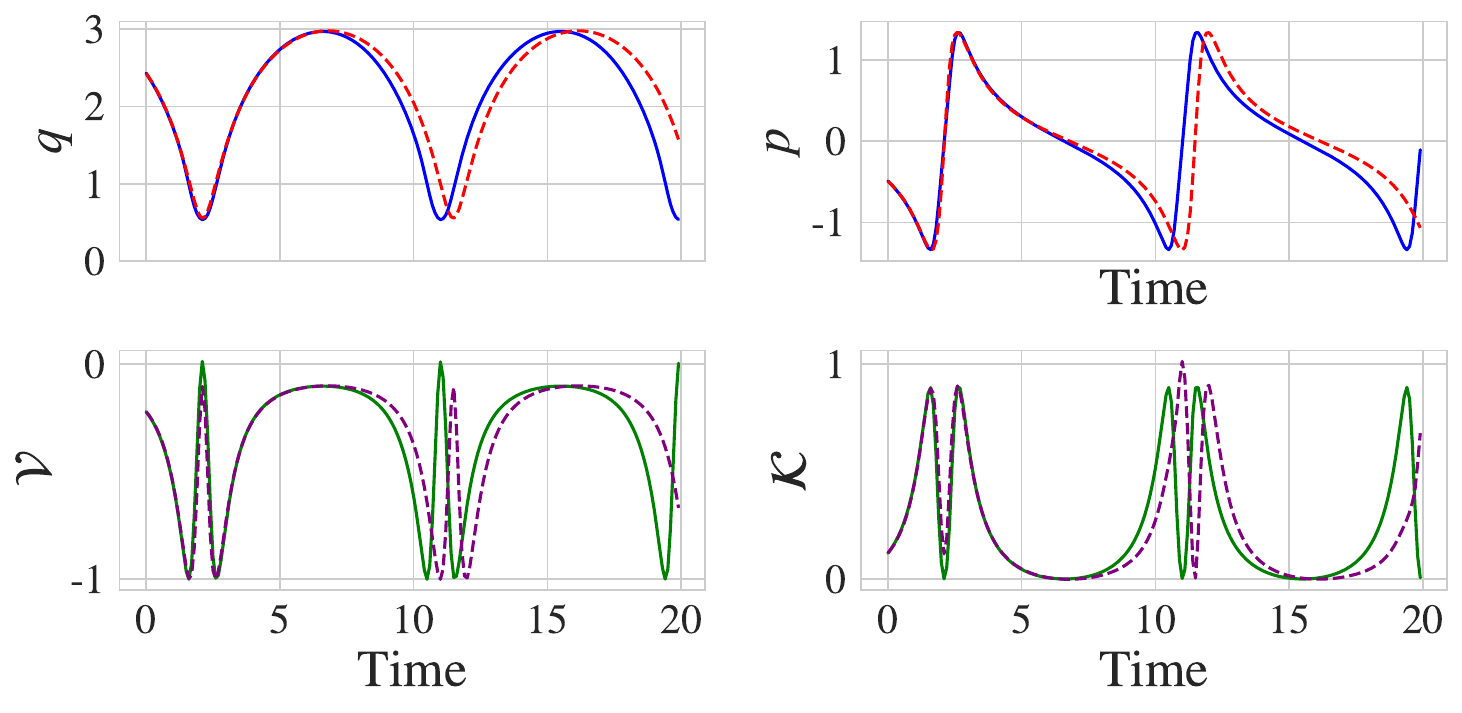}
    \caption{ASRNN predictions (trained with NSR 10\%, $\tau=25\Delta t$) compared to ground truth for the Morse potential with $\alpha=1.5$ (unseen). The solid lines represent the ASRNN predicted trajectory while the dashed lines are the ground truth---note the almost discontinuous change in momentum due to the `hard wall' type potential.}
    \label{fig:morse_example2}
\end{figure}

\subsubsection{Symbolic Regression}

As before, trained ASRNNs are used to generate regular trajectories for symbolic regression. Because PySINDy does not support exponential basis functions without prior specification, we apply PySR directly to the learned kinetic and potential energies. The regressor is trained on 2000 time-shuffled samples drawn from trajectories of length 500, generated from 40 initial conditions at fixed $\alpha$ but varying energies. While the kinetic energy function is learnt exactly, PySR cannot recover the exact exponential form of the Morse potential \footnote{This is not due to weaker predictions from the ASRNN but rather due to performance of the symbolic regression algorithms themselves, to confirm this we also attempted the symbolic regression with `ground truth' data and found the same results.}. However, when restricted to polynomial functions, it successfully learns an approximation corresponding to a truncation of the expansion \footnote{See Supplementary material for exact expressions},
\begin{equation}
    \mathcal{V}(\tilde{q}) = \sum_{n=0}^{\infty} c_n \tilde{q}^n =  2\sum_{n=0}^{\infty} (-1)^n (2^{n-1}-1) \frac{\alpha^n \tilde{q}^n}{n!}, \ \tilde{q} := q-1.
    \label{expanded_form}
\end{equation}
In particular, the regressor correctly identifies  $c_1 = 0$ in Eq.~\ref{expanded_form}
and we thus use the form $c_2 = \alpha^2$ to estimate the predicted value of $\alpha$ from the symbolic regression---see Table~\ref{tab:morse_SR} for the obtained parameters under various extents of noise.

\begin{table}[]
\centering
\begin{tabular}{|c|c|cccc|}
\hline
\multirow{2}{*}{\textbf{True}} &
  \multirow{2}{*}{\textbf{NSR}} &
  \multicolumn{4}{c|}{\textbf{$\tau$}} \\ \cline{3-6} 
 &  & \multicolumn{1}{c|}{\textbf{0.02}} & \multicolumn{1}{c|}{\textbf{0.1}} & \multicolumn{1}{c|}{\textbf{0.5}} & \textbf{2.5} \\ \hline
\multirow{3}{*}{\begin{tabular}[c]{@{}c@{}}2\\ (seen)\end{tabular}} &
  0\% &
  \multicolumn{4}{c|}{1.9995} \\ \cline{2-6} 
 &
  5\% &
  \multicolumn{1}{c|}{1.9927} &
  \multicolumn{1}{c|}{1.9905} &
  \multicolumn{1}{c|}{2.0437} &
  2.0816 \\ \cline{2-6} 
 &
  10\% &
  \multicolumn{1}{c|}{2.014} &
  \multicolumn{1}{c|}{2.0310} &
  \multicolumn{1}{c|}{2.0792} &
  2.1005 \\ \hline
\multirow{3}{*}{\begin{tabular}[c]{@{}c@{}}1.5\\ (unseen)\end{tabular}} &
  0\% &
  \multicolumn{4}{c|}{1.5263} \\ \cline{2-6} 
 &
  5\% &
  \multicolumn{1}{c|}{1.5234} &
  \multicolumn{1}{c|}{1.5390} &
  \multicolumn{1}{c|}{1.5467} &
  1.5754 \\ \cline{2-6} 
 &
  10\% &
  \multicolumn{1}{c|}{1.424} &
  \multicolumn{1}{c|}{1.3950} &
  \multicolumn{1}{c|}{1.5748} &
  1.5969 \\ \hline
\end{tabular}
\caption{Estimated parameter $\alpha$ obtained from symbolic regression with PySR for data generated by trained ASRNNs with various extents of noise for the Morse potential.}
\label{tab:morse_SR}
\end{table}

\section{Theoretical analysis} \label{theory}

The improved performance of ASRNNs in contrast to previous parameter-cognizant architectures and their behaviour under noise can be formalised theoretically. In practice, AHNNs require estimates of time derivatives during training (Eq.~\ref{HNNloss}), which are typically obtained via finite differences. In the presence of noise, such estimates suffer from a severe bias–variance trade-off: for example, the forward finite-difference estimator under OU correlated noise satisfies
\begin{equation}
    \mathrm{Var}\!\left(\frac{y_{t+\Delta s,i}-y_{t,i}}{\Delta s}\right)
= \frac{2\sigma_{\infty}^2}{\Delta s^2} (1 - e^{-\Delta s/\tau}),
\label{finite difference estimator}
    \end{equation}
where $\Delta s$ is the sampling interval of the data. Thus while small $\Delta s$ improves the point accuracy of the finite difference estimator, it results in diverging variance due to noise. \\

In contrast, learning under symplectic recurrent architectures is not strongly affected by the sampling interval and is purely determined by properties of the noise and dynamics, i.e.
\begin{theorem}
\label{theorem1}
    Let $\hat{z}_N(\theta;z_0 + \eta_0) = \Phi_{\theta, \Delta t}^{N} (z_0 + \eta_0) \in \mathbb{R}^d$ be an N-step ASRNN map with integrator time step $\Delta t$, model parameters $\theta$, initial (noiseless) observation $z_0$, and correlated, noisy observations $z_N+\eta_N$ where $\eta_N | \eta_0 \sim \mathcal{N}(a^N \eta_0, \sigma_{\infty}^2 (1-a^{2N}) I_d)$, the single sample loss gradient $\nabla_{\theta} l = \nabla_{\theta} (||\hat{z}_N(\theta;z_0 + \eta_0) - z_N - \eta_N||^2) $ satisfies
    \begin{enumerate}
        \item 
        \begin{multline}
            \mathbb{E}_{\eta_0, \eta_N} [\nabla_{\theta} l] = \nabla_{\theta}l_0 +\\ \sigma_{\infty}^2 \left [ \frac{1}{2} \nabla_{\theta} (\nabla_z^2 f_{\theta})|_{z_0} - \left(\nabla_z^2 \frac{\partial \hat{z}_N^T}{\partial\theta}\Bigg\vert_{z_0} \right) z_N \right]\\ - 2a^N \sigma_{\infty}^2\nabla_{\theta}(\nabla_z \cdot \hat{z}_N) |_{z_0} + \mathcal{O}(\sigma_{\infty}^3)
        \end{multline}
    \end{enumerate}
    where $l_0 (\theta ; z_0, z_N) = ||\hat{z}_N (\theta; z_0) - z_N||^2$ is the noise-free contribution to the loss and $f(\theta, z) = ||\hat{z}_N(\theta ; z)||^2$.
\end{theorem}
\begin{proof}
    See Appendix~\ref{appendix proof}
\end{proof}

Thus the expected per sample gradient is determined by the pure (no-noise) loss, $\mathcal{O}(\sigma_{\infty}^2)$ terms involving the Laplacians of the learned map $\hat{z}_N(\theta;z_0)$, and a term proportional to $e^{-N\Delta t/\tau}$. Note that this term is the only point where the sampling interval $N\Delta t$ is involved, thus if $N\Delta t >> \tau \Rightarrow e^{-N\Delta t/\tau} \to 0$ and the effect of time correlations vanishes, whereas in the regime of large $\tau$ the effect of correlations offsets the expected loss gradient by a term proportional to trace of the Jacobian of the learned transformation from $z_0 \to z_N$, i.e. $(\nabla_z \cdot \hat{z}_N) |_{z_0}$. This explains the model behaviour on increasing $\tau$ as observed in Section \ref{results}. \\

Note that for AHNN type architectures the expected loss gradient under OU noise additionally depends inversely on the sampling interval (see Appendix~\ref{AHNN theorem} for proof) echoing the behaviour of the finite difference estimator, Eq. \ref{finite difference estimator}.

\section{Conclusion} \label{conclusion}

In this work we introduced Adaptable Symplectic Recurrent Neural Networks (ASRNNs), a parameter-cognizant architecture for learning Hamiltonian dynamics by combining adaptable Hamiltonian neural networks with symplectic recurrent integration. This work serves as an empirical study demonstrating the ability of such architectures to learn complex Hamiltonians under the regime of extremely sparse, irregularly spaced and noisy data. We support this claim with theoretical arguments to explain why symplectic models are more robust to training under noise in comparison to previous parameter-cognizant models requiring time derivatives for training. The analysis further explains the empirical behaviour of ASRNNs under various extents of correlated noise. \\

Remarkably, ASRNNs achieve accurate long-term predictions even when each training trajectory consists of only two time points, potentially irregularly spaced and corrupted by correlated noise. Leveraging this capability, we used ASRNNs as structure-preserving data generators to enable symbolic regression with PySINDy and PySR. This pipeline successfully recovered the exact equations of motion for the Henon–Heiles system and accurate polynomial approximations to the Morse Hamiltonian, despite the severe sparsity of the original data. The data regime explored here is sparse along two `axes': the number of available observations per parameter setting and the number of samples within each observation window. That ASRNNs perform well under such constraints, using relatively small networks, highlights their robustness and suitability for realistic scientific settings where dense, noise-free measurements are unavailable. As the amount of data is increased along any of the two axes, we expect a much more robust performance. However, the aim of this study is to illustrate the performance of a small model, trained on very sparse data, under the ASRNN inductive bias.

These results suggest that ASRNNs provide a practical route to interpretable discovery of Hamiltonian dynamics from limited experimental or observational data. Future work includes studying higher-dimensional systems, learning dynamics in noncanonical coordinates (similar to \cite{choudhary2021forecasting}), and incorporating initial-state optimisation to further mitigate the impact of noise in observed initial conditions.

\section*{Software and Data}

Code to reproduce the results in this paper as well as produce additional results can be found at \href{https://github.com/fibrebundle/ASRNN_Sparse_Data.git}{https://github.com/fibrebundle/ASRNN\_Sparse\_Data.git}. Full symbolic regression results can be found in the Supplementary Material.

\begin{acknowledgments}
VT acknowledges the Rhodes Trust and Mathematical Institute, University of Oxford for funding.
\end{acknowledgments}

\appendix
%\begin{comment}
    
\section{Extrapolation beyond the training region} \label{appendix1}

The primary focus of our empirical study has been on parameter adaptability within the limits of the training region, for example $[0.2, 0.8]^2$ in the case of the Henon-Heiles potential. However, it is equally interesting to study the performance of ASRNNs when tested beyond this region, i.e. ability to `extrapolate' to qualitatively different dynamical regimes. To illustrate this, we trained an ASRNN to learn the dynamics of the double well (DW) potential given as
\begin{equation}
    \mathcal{V}(q) = \frac{\alpha}{2} q^2 + \frac{1}{4} q^4.
\end{equation}
Here $\alpha$ acts as a bifurcation parameter: for $\alpha>0$ the potential has a single minimum while with $\alpha<0$ there are two degenerate minima. Thus the DW system displays  symmetry breaking with qualitatively different dynamics on either side of $\alpha=0$ making it a simple but non-trivial model to study model extrapolation. \\

We trained models with the same architecture as that for Section~\ref{Morse} with sparse data for only positive parameter values, these being $\alpha \in \{0.1, 0.3, 0.5, 0.7, 0.9\}$. Data was generated as described in Section~\ref{TD} After training, we tested what the potential energy network $\mathcal{V}_{\theta}$ had learned about negative parameter values, having been trained only on positive values. While the ASRNNs performed excellently within the training region $[0.1, 0.9]$ even in the presence of noise, the behaviour for $\alpha \leq 0$ was sensitive to both noise and weight initialisation. Models trained on noise-free data were able to predict (at least qualitatively) symmetry breaking for negative parameter values---see Fig.~\ref{fig:DW_results}. In contrast, performance for models trained under noise degraded sharply outside the training region and they struggled to capture this behaviour. This in itself is not surprising as it is well known that MLPs struggle to extrapolate \cite{xu2021how}. However, it is notable that under noise-free conditions, ASRNNs appear to capture the essential transition at $\alpha=0$ and predict well for small negative $\alpha$.\\

\begin{figure*}
    \centering
    \includegraphics[width=\linewidth]{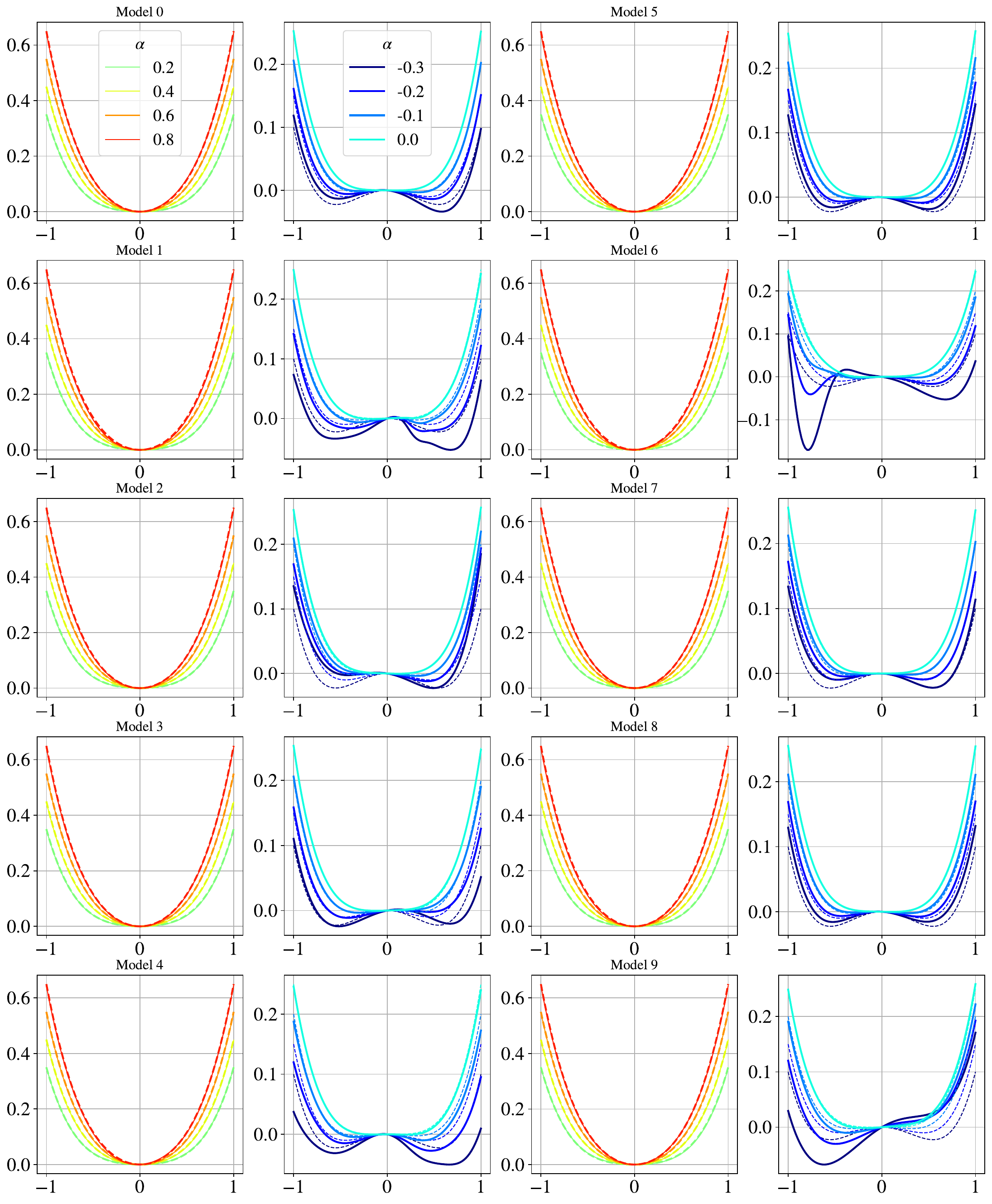}
    \caption{Learned (solid line) and true (dashed line) potential profiles for the double well potential for a sample of 10 models. For each model the left panel considers unseen values of $\alpha$ in the training region (single minimum), right panel displays the bifurcation point $\alpha=0$ and other $\alpha<0$ (two minima).}
    \label{fig:DW_results}
\end{figure*}

As observed from Fig.~\ref{fig:DW_results}, independent models trained on the same data but differing in the random weight initialization (and with very similar training and validation loss profiles), made significantly different predictions in the bifurcation regime. Thus as a direction of future research, it would be interesting to explore how such `equivalent' models (in the sense of identical training behaviour) differ internally, and since all one has access to is the loss profile, how should practitioners judge whether a model has captured the underlying physics of the system.

\section{Proof of Theorem \ref{theorem1}} \label{appendix proof}

\begin{proof}
    Let $\hat{z}_N(\theta;z_0 + \eta_0) = \Phi_{\theta, \Delta t}^{N} (z_0 + \eta_0) \in \mathbb{R}^d$ be an N-step ASRNN map with model parameters $\theta$, integrator time step $\Delta t$, initial (noiseless) observation $z_0=(q_0, p_0)$, and for correlated, noisy observations $z_N+\eta_N$ where $\eta_N | \eta_0 \sim \mathcal{N}(a^N \eta_0, \sigma_{\infty}^2 (1-a^{2N}) I_d)$. The per-sample loss is given by
    \begin{equation}
        l(\theta ; z_0, \eta_0, \eta_N) = ||\hat{z}_N (\theta;z_0 + \eta_0) - z_N - \eta_N||^2,
    \end{equation}
    and the gradient may be written as
    \begin{equation}
        \nabla_{\theta} l = 2 \frac{\partial \hat{z}_N^T}{\partial \theta} (\theta ; z_0+\eta_0) \left[\hat{z}_N (\theta;z_0 + \eta_0) - z_N - \eta_N \right]
    \end{equation}
    Defining $M_\theta(z) = \nabla_{\theta}\hat{z}_N|_z$ and expanding to second order in $\eta_0$ yields
    \begin{multline}
        \nabla_{\theta} l  
            =  2 \Biggr[M_\theta(z_0) + \sum_i \eta_{0i}\frac{\partial M_\theta}{\partial z_i}  (z_0)\\
            +\frac{1}{2} \sum_{i,j} \eta_{0i} \eta_{0j} \frac{\partial M_\theta}{\partial z_i \partial z_j} (z_0)\Biggr]^T \Biggr[  \hat{z}_N (\theta;z_0) +  \sum_i \eta_{0i}\frac{\partial \hat{z}_N}{\partial z_i}  (z_0)\\
            +\frac{1}{2} \sum_{i,j} \eta_{0i} \eta_{0j} \frac{\partial \hat{z}_N}{\partial z_i \partial z_j} (z_0)   - z_N - \eta_N\Biggr] + \mathcal{O}(||\eta_0||^3)
            \label{B3}
    \end{multline}
    With $\mathbb{E}[\eta_{0i}] = 0$, $\mathbb{E}[\eta_{0i} \eta_{0j}] = \sigma_{\infty}^2 \delta_{ij}$ and $\mathbb{E}[\eta_N|\eta_0] = a^N\eta_0$, the expected loss gradient may be written as
    \begin{multline}
        \mathbb{E}_{\eta_N,\eta_0} \left[\nabla_{\theta} l (\theta; z_0, \eta_0, \eta_N \right] = \mathbb{E}_{\eta_0}\left[\mathbb{E}_{\eta_N|\eta_0} \left[\nabla_{\theta} l \right]\right]= \nabla_{\theta}l_0 +\\ \sigma_{\infty}^2 \left [ \frac{1}{2} \nabla_{\theta} (\nabla_z^2 f_{\theta})|_{z_0} - \left(\nabla_z^2 M_{\theta}(z_0)\right) z_N \right]\\ - 2a^N \sigma_{\infty}^2\nabla_{\theta}(\nabla_z \cdot \hat{z}_N) |_{z_0} + \mathcal{O}(\sigma_{\infty}^3).
    \end{multline}
    Here we have identified $\nabla_{\theta} l_0 = 2M_\theta^T(z_0) (\hat{z}_N(\theta;z_0) -z_N)$ as the noise-free gradient.
\end{proof}

\subsection{Expected loss gradient for AHNNs} \label{AHNN theorem}

We can derive an analogue of Theorem~\ref{theorem1} for an adaptable AHNN to illustrate the bias-variance tradeoff of using finite difference estimators for training AHNNs in noisy conditions.

\begin{theorem}
\label{theorem2}
    Let $\hat{\dot{z}}(\theta;z) = \left(\partial_p \mathcal{H}_{\theta} (z;\lambda), -\partial_q \mathcal{H}_{\theta} (z;\lambda) \right)$ denote the AHNN vector field with model parameters $\theta$. Let the initial (noiseless) observation be $z_0$, and correlated, noisy observations $z_N+\eta_N$ where $\eta_N | \eta_0 \sim \mathcal{N}(a^N \eta_0, \sigma_{\infty}^2 (1-a^{2N}) I_d)$, the single sample loss gradient $\nabla_{\theta} l = \nabla_{\theta} (||\hat{\dot{z}}(\theta;z_0 + \eta_0) - \frac{(z_N + \eta_N) - (z_0 + \eta_0)}{\Delta s}||^2) $ satisfies
    \begin{enumerate}
        \item 
        \begin{multline}
            \mathbb{E}_{\eta_0, \eta_N} [\nabla_{\theta} l] = \nabla_{\theta}l_{0, \text{AHNN}} +\\ \sigma_{\infty}^2 \left [ \frac{1}{2} \nabla_{\theta} (\nabla_z^2 g_{\theta})|_{z_0} - \left(\nabla_z^2 \frac{\partial \hat{\dot{z}}^T}{\partial\theta}\Bigg\vert_{z_0} \right) \frac{z_N - z_0}{\Delta s} \right]\\ + 2\frac{(1-a^N) \sigma_{\infty}^2}{\Delta s}\nabla_{\theta}(\nabla_z \cdot \hat{\dot{z}}) |_{z_0} + \mathcal{O}(\sigma_{\infty}^3),
        \end{multline}
    \end{enumerate}
    where $\Delta s$ is the sampling interval, $l_{0,\text{AHNN}} (\theta ; z_0, z_N) = ||\hat{\dot{z}} (\theta; z_0) - \frac{z_N - z_0}{\Delta s}||^2$ is the noise-free contribution to the loss and $g(\theta, z) = ||\hat{\dot{z}}(\theta ; z)||^2$.
\end{theorem}
\begin{proof}
    The proof follows the same procedure as Theorem \ref{theorem1}: \\

    Let $\hat{\dot{z}}(\theta;z) = \left(\partial_p \mathcal{H}_{\theta} (z;\lambda), -\partial_q \mathcal{H}_{\theta} (z;\lambda) \right)$ denote the AHNN vector field with model parameters $\theta$. Let the initial (noiseless) observation be $z_0$, and correlated, noisy observations $z_N+\eta_N$ where $\eta_N | \eta_0 \sim \mathcal{N}(a^N \eta_0, \sigma_{\infty}^2 (1-a^{2N}) I_d)$. The per sample loss is given by 
    \begin{equation}
        l(\theta;z_0, z_N) = \Bigg|\Bigg|\hat{\dot{z}}(\theta;z_0 + \eta_0) - \frac{(z_N + \eta_N) - (z_0 + \eta_0)}{\Delta s}\Bigg|\Bigg|^2
    \end{equation}
    and the gradient may be written as
    \begin{equation}
        \nabla_{\theta} l = 2 \frac{\partial \hat{\dot{z}}^T}{\partial \theta} (\theta ; z_0+\eta_0) \left[\hat{\dot{z}} (\theta;z_0 + \eta_0) - \frac{z_N + \eta_N - z_0 - \eta_0}{\Delta s} \right]
    \end{equation}
    Defining $F_\theta(z) = \nabla_{\theta}\hat{\dot{z}}|_z$ and expanding to second order in $\eta_0$ yields
    \begin{multline}
        \nabla_{\theta} l  
            =  2 \Biggr[F_\theta(z_0) + \sum_i \eta_{0i}\frac{\partial F_\theta}{\partial z_i}  (z_0)\\
            +\frac{1}{2} \sum_{i,j} \eta_{0i} \eta_{0j} \frac{\partial F_\theta}{\partial z_i \partial z_j} (z_0)\Biggr]^T \Biggr[  \hat{\dot{z}} (\theta;z_0) +  \sum_i \eta_{0i}\frac{\partial \hat{\dot{z}}}{\partial z_i}  (z_0)\\
            +\frac{1}{2} \sum_{i,j} \eta_{0i} \eta_{0j} \frac{\partial \hat{\dot{z}}}{\partial z_i \partial z_j} (z_0)  - \frac{z_N - z_0}{\Delta s} - \frac{\eta_N- \eta_0}{\Delta s}\Biggr] + \mathcal{O}(||\eta_0||^3)
            \label{B3}
    \end{multline}
    With $\mathbb{E}[\eta_{0i}] = 0$, $\mathbb{E}[\eta_{0i} \eta_{0j}] = \sigma_{\infty}^2 \delta_{ij}$ and $\mathbb{E}[\eta_N|\eta_0] = a^N\eta_0$, the expected loss gradient may be written as
    \begin{multline}
        \mathbb{E}_{\eta_N,\eta_0} \left[\nabla_{\theta} l (\theta; z_0, \eta_0, \eta_N \right] = \mathbb{E}_{\eta_0}\left[\mathbb{E}_{\eta_N|\eta_0} \left[\nabla_{\theta} l \right]\right]= \nabla_{\theta}l_{0,\text{AHNN}} +\\ \sigma_{\infty}^2 \left [ \frac{1}{2} \nabla_{\theta} (\nabla_z^2 g_{\theta})|_{z_0} - (\nabla_z^2 F_{\theta}(z_0)) \frac{z_N - z_0}{\Delta s} \right]\\ + 2\frac{(1-a^N) \sigma_{\infty}^2}{\Delta s}\nabla_{\theta}(\nabla_z \cdot \hat{\dot{z}}) |_{z_0} + \mathcal{O}(\sigma_{\infty}^3)
    \end{multline}
    Here we have identified $\nabla_{\theta, \text{AHNN}} l_0 = 2F_\theta^T(z_0) (\hat{\dot{z}}(\theta;z_0) -\frac{z_N-z_0}{\Delta s})$ as the noise-free gradient.
\end{proof}

We can observe that in contrast to Theorem~\ref{theorem1}, the expected gradient for AHNNs scales inversely with the sampling interval thus making such models an impractical choice for sparse, noisy data.

\bibliography{apssamp}% Produces the bibliography via BibTeX.

@article{SciNet, title={Discovering physical concepts with neural networks}, volume={124}, ISSN={0031-9007, 1079-7114}, DOI={10.1103/PhysRevLett.124.010508}, abstractNote={Despite the success of neural networks at solving concrete physics problems, their use as a general-purpose tool for scientific discovery is still in its infancy. Here, we approach this problem by modelling a neural network architecture after the human physical reasoning process, which has similarities to representation learning. This allows us to make progress towards the long-term goal of machine-assisted scientific discovery from experimental data without making prior assumptions about the system. We apply this method to toy examples and show that the network finds the physically relevant parameters, exploits conservation laws to make predictions, and can help to gain conceptual insights, e.g. Copernicus’ conclusion that the solar system is heliocentric.}, number={1}, journal={Physical Review Letters}, author={Iten, Raban and Metger, Tony and Wilming, Henrik and del Rio, Lidia and Renner, Renato}, year={2020}, month=jan, pages={010508} }

@article{PINNsReview, title={Physics-informed machine learning},
journal = {Nature Reviews Physics},
DOI={10.1038/s42254-021-00314-5},
author={Karniadakis, George and Kevrekidis, Yannis and Lu, Lu and Perdikaris, Paris and Wang, Sifan and Yang, Liu}, year={2021}, volume={3}, pages={422--440}}

@article{AINewton,
doi = {10.1088/2632-2153/acfa63},
year = {2023},
month = {oct},
publisher = {IOP Publishing},
volume = {4},
number = {4},
pages = {045002},
author = {Lemos, Pablo and Jeffrey, Niall and Cranmer, Miles and Ho, Shirley and Battaglia, Peter},
title = {Rediscovering orbital mechanics with machine learning},
journal = {Machine Learning: Science and Technology}
}

@article{BayesianScientist, title={A {B}ayesian machine scientist to aid in the solution of challenging scientific problems}, volume={6}, DOI={10.1126/sciadv.aav6971}, abstractNote={Closed-form, interpretable mathematical models have been instrumental for advancing our understanding of the world; with the data revolution, we may now be in a position to uncover new such models for many systems from physics to the social sciences. However, to deal with increasing amounts of data, we need “machine scientists” that are able to extract these models automatically from data. Here, we introduce a Bayesian machine scientist, which establishes the plausibility of models using explicit approximations to the exact marginal posterior over models and establishes its prior expectations about models by learning from a large empirical corpus of mathematical expressions. It explores the space of models using Markov chain Monte Carlo. We show that this approach uncovers accurate models for synthetic and real data and provides out-of-sample predictions that are more accurate than those of existing approaches and of other nonparametric methods.}, number={5}, journal={Science Advances}, author={Guimerà, Roger and Reichardt, Ignasi and Aguilar-Mogas, Antoni and Massucci, Francesco A. and Miranda, Manuel and Pallarès, Jordi and Sales-Pardo, Marta}, year={2020}, month=jan, pages={eaav6971} }

@article{StochasticComplexSYstems, title={Learning interpretable dynamics of stochastic complex systems from experimental data}, volume={15}, rights={2024 The Author(s)}, ISSN={2041-1723}, DOI={10.1038/s41467-024-50378-x}, abstractNote={Complex systems with many interacting nodes are inherently stochastic and best described by stochastic differential equations. Despite increasing observation data, inferring these equations from empirical data remains challenging. Here, we propose the Langevin graph network approach to learn the hidden stochastic differential equations of complex networked systems, outperforming five state-of-the-art methods. We apply our approach to two real systems: bird flock movement and tau pathology diffusion in brains. The inferred equation for bird flocks closely resembles the second-order Vicsek model, providing unprecedented evidence that the Vicsek model captures genuine flocking dynamics. Moreover, our approach uncovers the governing equation for the spread of abnormal tau proteins in mouse brains, enabling early prediction of tau occupation in each brain region and revealing distinct pathology dynamics in mutant mice. By learning interpretable stochastic dynamics of complex systems, our findings open new avenues for downstream applications such as control.}, number={1}, journal={Nature Communications}, author={Gao, Ting-Ting and Barzel, Baruch and Yan, Gang}, year={2024}, month=jul, pages={6029}}

@article{AIPoincare,
  title = {Machine Learning Conservation Laws from Trajectories},
  author = {Liu, Ziming and Tegmark, Max},
  journal = {Physical Review Letters},
  volume = {126},
  issue = {18},
  pages = {180604},
  numpages = {6},
  year = {2021},
  month = {May},
  publisher = {American Physical Society},
  doi = {10.1103/PhysRevLett.126.180604},
}

@article{
AIFeynman,
author = {Silviu-Marian Udrescu  and Max Tegmark },
title = {A{I} {F}eynman: A physics-inspired method for symbolic regression},
journal = {Science Advances},
volume = {6},
number = {16},
pages = {eaay2631},
year = {2020},
doi = {10.1126/sciadv.aay2631}}

@article{wang2023scientific,
  title={Scientific discovery in the age of artificial intelligence},
  author={Wang, Hanchen and Fu, Tianfan and Du, Yuanqi and Gao, Wenhao and Huang, Kexin and Liu, Ziming and Chandak, Payal and Liu, Shengchao and Van Katwyk, Peter and Deac, Andreea and others},
  journal={Nature},
  volume={620},
  number={7972},
  pages={47--60},
  year={2023},
  publisher={Nature Publishing Group UK London},
  url = {https://doi.org/10.1038/s41586-023-06221-2}
}

@article{angelis2023artificial,
  title={Artificial intelligence in physical sciences: Symbolic regression trends and perspectives},
  author={Angelis, Dimitrios and Sofos, Filippos and Karakasidis, Theodoros E},
  journal={Archives of Computational Methods in Engineering},
  volume={30},
  number={6},
  pages={3845--3865},
  year={2023},
  publisher={Springer},
  url = {https://doi.org/10.1007/s11831-023-09922-z}
}

@article{makke2024interpretable,
  title={Interpretable scientific discovery with symbolic regression: a review},
  author={Makke, Nour and Chawla, Sanjay},
  journal={Artificial Intelligence Review},
  volume={57},
  number={1},
  pages={2},
  year={2024},
  publisher={Springer},
  url = {https://doi.org/10.1007/s10462-023-10622-0}
}

@article{xu2021artificial,
  title={Artificial intelligence: A powerful paradigm for scientific research},
  author={Xu, Yongjun and Liu, Xin and Cao, Xin and Huang, Changping and Liu, Enke and Qian, Sen and Liu, Xingchen and Wu, Yanjun and Dong, Fengliang and Qiu, Cheng-Wei and others},
  journal={The Innovation},
  volume={2},
  number={4},
  year={2021},
  publisher={Elsevier},
  url = {https://doi.org/10.1016/j.xinn.2021.100179}
}

@article{alphafold,
  title={Highly accurate protein structure prediction with {A}lpha{F}old},
  author={Jumper, John and Evans, Richard and Pritzel, Alexander and Green, Tim and Figurnov, Michael and Ronneberger, Olaf and Tunyasuvunakool, Kathryn and Bates, Russ and {\v{Z}}{\'\i}dek, Augustin and Potapenko, Anna and others},
  journal={Nature},
  volume={596},
  number={7873},
  pages={583--589},
  year={2021},
  publisher={Nature Publishing Group},
  url = {https://doi.org/10.1038/s41586-021-03819-2}
}

@article{merchant2023scaling,
  title={Scaling deep learning for materials discovery},
  author={Merchant, Amil and Batzner, Simon and Schoenholz, Samuel S and Aykol, Muratahan and Cheon, Gowoon and Cubuk, Ekin Dogus},
  journal={Nature},
  volume={624},
  number={7990},
  pages={80--85},
  year={2023},
  publisher={Nature Publishing Group UK London},
  url = {https://doi.org/10.1038/s41586-023-06735-9}
}

@article{pavone2023machine,
  title={Machine learning and {B}ayesian inference in nuclear fusion research: an overview},
  author={Pavone, A and Merlo, A and Kwak, S and Svensson, J},
  journal={Plasma Physics and Controlled Fusion},
  volume={65},
  number={5},
  pages={053001},
  year={2023},
  publisher={IOP Publishing},
  url = {https://doi.org/10.1088/1361-6587/acc60f}
}

@article{arhouni2025artificial,
  title={Artificial intelligence-driven advances in nuclear technology: Exploring innovations, applications, and future prospects},
  author={Arhouni, Fatima Ezzahra and Abdo, Maged Ahmed Saleh and Ouakkas, Saad and Bouhssa, Mohamed Lhadi and Boukhair, Aziz},
  journal={Annals of Nuclear Energy},
  volume={213},
  pages={111151},
  year={2025},
  publisher={Elsevier},
  url = {https://doi.org/10.1016/j.anucene.2024.111151}
}

@article{bracco2024machine,
  title={Machine learning for the physics of climate},
  author={Bracco, Annalisa and Brajard, Julien and Dijkstra, Henk A and Hassanzadeh, Pedram and Lessig, Christian and Monteleoni, Claire},
  journal={Nature Reviews Physics},
  volume = {7},
  pages={6--20},
  year={2025},
  publisher={Nature Publishing Group UK London},
  url = {https://doi.org/10.1038/s42254-024-00776-3}
}

@article{lai2024machine,
  title={Machine learning for climate physics and simulations},
  author={Lai, Ching-Yao and Hassanzadeh, Pedram and Sheshadri, Aditi and Sonnewald, Maike and Ferrari, Raffaele and Balaji, Venkatramani},
  journal={Annual Review of Condensed Matter Physics},
  volume={16},
  pages = {343--365},
  year={2024},
  publisher={Annual Reviews},
  url = {https://doi.org/10.1146/annurev-conmatphys-043024-114758}
}

@article{
graphcast,
author = {Remi Lam  and Alvaro Sanchez-Gonzalez  and Matthew Willson  and Peter Wirnsberger  and Meire Fortunato  and Ferran Alet  and Suman Ravuri  and Timo Ewalds  and Zach Eaton-Rosen  and Weihua Hu  and Alexander Merose  and Stephan Hoyer  and George Holland  and Oriol Vinyals  and Jacklynn Stott  and Alexander Pritzel  and Shakir Mohamed  and Peter Battaglia },
title = {Learning skillful medium-range global weather forecasting},
journal = {Science},
volume = {382},
number = {6677},
pages = {1416-1421},
year = {2023},
doi = {10.1126/science.adi2336}}

@article{eyring2024pushing,
  title={Pushing the frontiers in climate modelling and analysis with machine learning},
  author={Eyring, Veronika and Collins, William D and Gentine, Pierre and Barnes, Elizabeth A and Barreiro, Marcelo and Beucler, Tom and Bocquet, Marc and Bretherton, Christopher S and Christensen, Hannah M and Dagon, Katherine and others},
  journal={Nature Climate Change},
  volume={14},
  number={9},
  pages={916--928},
  year={2024},
  publisher={Nature Publishing Group UK London},
  url = {https://doi.org/10.1038/s41558-024-02095-y}
}

@article{kochkov2021machine,
  title={Machine learning--accelerated computational fluid dynamics},
  author={Kochkov, Dmitrii and Smith, Jamie A and Alieva, Ayya and Wang, Qing and Brenner, Michael P and Hoyer, Stephan},
  journal={Proceedings of the National Academy of Sciences},
  volume={118},
  number={21},
  pages={e2101784118},
  year={2021},
  publisher={National Acad Sciences},
  url = {https://doi.org/10.1073/pnas.2101784118}
}

@article{vinuesa2022enhancing,
  title={Enhancing computational fluid dynamics with machine learning},
  author={Vinuesa, Ricardo and Brunton, Steven L},
  journal={Nature Computational Science},
  volume={2},
  number={6},
  pages={358--366},
  year={2022},
  publisher={Nature Publishing Group US New York},
  url = {https://doi.org/10.1038/s43588-022-00264-7}
}

@article{brunton2024promising,
  title={Promising directions of machine learning for partial differential equations},
  author={Brunton, Steven L and Kutz, J Nathan},
  journal={Nature Computational Science},
  volume={4},
  number={7},
  pages={483--494},
  year={2024},
  publisher={Nature Publishing Group US New York},
  url = {https://doi.org/10.1038/s43588-024-00643-2}
}

@article{pravsnikar2024machine,
  title={Machine learning heralding a new development phase in molecular dynamics simulations},
  author={Pra{\v{s}}nikar, Eva and Ljubi{\v{c}}, Martin and Perdih, Andrej and Bori{\v{s}}ek, Jure},
  journal={Artificial intelligence review},
  volume={57},
  number={4},
  pages={102},
  year={2024},
  publisher={Springer},
  url = {https://doi.org/10.1007/s10462-024-10731-4}
}

@Article{toscano2024pinnspikansrecentadvances,
author={Toscano, Juan Diego
and Oommen, Vivek
and Varghese, Alan John
and Zou, Zongren
and Ahmadi Daryakenari, Nazanin
and Wu, Chenxi
and Karniadakis, George Em},
title={From {PINNs} to {PIKAN}s: {R}ecent advances in physics-informed machine learning},
journal={Machine Learning for Computational Science and Engineering},
year={2025},
month={Mar},
day={11},
volume={1},
number={1},
pages={15},
issn={3005-1436},
doi={10.1007/s44379-025-00015-1},
url={https://doi.org/10.1007/s44379-025-00015-1}
}

@book{goldstein2011classical,
  author    = {Herbert Goldstein and Charles P. Poole and John L. Safko},
  title     = {Classical Mechanics},
  edition   = {3rd},
  publisher = {Addison Wesley},
  year      = {2002},
  address   = {San Francisco, CA, USA},
  isbn      = {0-201-65702-3, 978-0-201-65702-9}
}

@inproceedings{greydanus2019hamiltonian,
 author = {Greydanus, Samuel and Dzamba, Misko and Yosinski, Jason},
 booktitle = {Advances in Neural Information Processing Systems},
 title = {Hamiltonian Neural Networks},
 url = {https://proceedings.neurips.cc/paper_files/paper/2019/file/26cd8ecadce0d4efd6cc8a8725cbd1f8-Paper.pdf},
 volume = {32},
 year = {2019}
}

@article{Han_Glaz_Haile_Lai_2021, title={Adaptable {H}amiltonian neural networks}, volume={3}, ISSN={2643-1564}, DOI={10.1103/PhysRevResearch.3.023156}, abstractNote={The rapid growth of research in exploiting machine learning to predict chaotic systems has revived a recent interest in Hamiltonian Neural Networks (HNNs) with physical constraints defined by the Hamilton’s equations of motion, which represent a major class of physics-enhanced neural networks. We introduce a class of HNNs capable of adaptable prediction of nonlinear physical systems: by training the neural network based on time series from a small number of bifurcation-parameter values of the target Hamiltonian system, the HNN can predict the dynamical states at other parameter values, where the network has not been exposed to any information about the system at these parameter values. The architecture of the HNN differs from the previous ones in that we incorporate an input parameter channel, rendering the HNN parameter--cognizant. We demonstrate, using paradigmatic Hamiltonian systems, that training the HNN using time series from as few as four parameter values bestows the neural machine with the ability to predict the state of the target system in an entire parameter interval. Utilizing the ensemble maximum Lyapunov exponent and the alignment index as indicators, we show that our parameter-cognizant HNN can successfully predict the route of transition to chaos. Physics-enhanced machine learning is a forefront area of research, and our adaptable HNNs provide an approach to understanding machine learning with broad applications.}, number={2}, journal={Physical Review Research}, author={Han, Chen-Di and Glaz, Bryan and Haile, Mulugeta and Lai, Ying-Cheng}, year={2021}, month=may, pages={023156} }

@inproceedings{
Chen2020Symplectic,
title={Symplectic Recurrent Neural Networks},
author={Zhengdao Chen and Jianyu Zhang and Martin Arjovsky and Léon Bottou},
booktitle={International Conference on Learning Representations},
year={2020},
url = {https://openreview.net/forum?id=BkgYPREtPr}
}

@inproceedings{
xiong2022nonseparablesymplecticneuralnetworks,
title={Nonseparable Symplectic Neural Networks},
author={Shiying Xiong and Yunjin Tong and Xingzhe He and Shuqi Yang and Cheng Yang and Bo Zhu},
booktitle={International Conference on Learning Representations},
year={2021},
url={https://openreview.net/forum?id=B5VvQrI49Pa}
}

@inproceedings{
cranmer2020lagrangianneuralnetworks,
title={Lagrangian Neural Networks},
author={Miles Cranmer and Sam Greydanus and Stephan Hoyer and Peter Battaglia and David Spergel and Shirley Ho},
booktitle={ICLR 2020 Workshop on Integration of Deep Neural Models and Differential Equations},
year={2020},
url={https://openreview.net/forum?id=iE8tFa4Nq}
}

@article{orderandchaos,
  title = {Physics-enhanced neural networks learn order and chaos},
  author = {Choudhary, Anshul and Lindner, John F. and Holliday, Elliott G. and Miller, Scott T. and Sinha, Sudeshna and Ditto, William L.},
  journal = {Physical Review E},
  volume = {101},
  issue = {6},
  pages = {062207},
  numpages = {8},
  year = {2020},
  month = {Jun},
  publisher = {American Physical Society},
  doi = {10.1103/PhysRevE.101.062207},
}

@article{choudhary2021forecasting,
  title={Forecasting {H}amiltonian dynamics without canonical coordinates},
  author={Choudhary, Anshul and Lindner, John F and Holliday, Elliott G and Miller, Scott T and Sinha, Sudeshna and Ditto, William L},
  journal={Nonlinear Dynamics},
  volume={103},
  pages={1553--1562},
  year={2021},
  publisher={Springer},
  url = {https://doi.org/10.1007/s11071-020-06185-2}
}

@inproceedings{contraints,
 author = {Finzi, Marc and Wang, Ke Alexander and Wilson, Andrew G},
 booktitle = {Advances in Neural Information Processing Systems},
 title = {Simplifying {H}amiltonian and {L}agrangian Neural Networks via Explicit Constraints},
 volume = {33},
 year = {2020},
 url = {https://proceedings.neurips.cc/paper/2020/file/9f655cc8884fda7ad6d8a6fb15cc001e-Paper.pdf}
}

@article{sosanya2022dissipative,
  title={Dissipative {H}amiltonian neural networks: Learning dissipative and conservative dynamics separately},
  author={Sosanya, Andrew and Greydanus, Sam},
  journal={arXiv preprint arXiv:2201.10085},
  year={2022},
  url = {https://doi.org/10.48550/arXiv.2201.10085}
}

@article{SympNets,
title = {{S}ymp{N}ets: Intrinsic structure-preserving symplectic networks for identifying {H}amiltonian systems},
journal = {Neural Networks},
volume = {132},
pages = {166-179},
year = {2020},
issn = {0893-6080},
doi = {https://doi.org/10.1016/j.neunet.2020.08.017},
author = {Pengzhan Jin and Zhen Zhang and Aiqing Zhu and Yifa Tang and George Em Karniadakis}
}

@article{HORN2025113536,
title = {A generalized framework of neural networks for {H}amiltonian systems},
journal = {Journal of Computational Physics},
volume = {521},
pages = {113536},
year = {2025},
issn = {0021-9991},
doi = {https://doi.org/10.1016/j.jcp.2024.113536},
author = {Philipp Horn and Veronica {Saz Ulibarrena} and Barry Koren and Simon {Portegies Zwart}}
}

@article{allen2023interpretable,
  title={Interpretable machine learning for discovery: Statistical challenges and opportunities},
  author={Allen, Genevera I and Gan, Luqin and Zheng, Lili},
  journal={Annual Review of Statistics and its Application},
  volume={11},
  pages = {97--121},
  year={2023},
  publisher={Annual Reviews},
  url = {https://doi.org/10.1146/annurev-statistics-040120-030919}
}

@article{brunton2016discovering,
  title={Discovering governing equations from data by sparse identification of nonlinear dynamical systems},
  author={Brunton, Steven L and Proctor, Joshua L and Kutz, J Nathan},
  journal={Proceedings of the National Academy of Sciences},
  volume={113},
  number={15},
  pages={3932--3937},
  year={2016},
  publisher={National Acad Sciences},
  url = {https://doi.org/10.1073/pnas.1517384113}
}

@article{cranmer2023interpretable,
  title={Interpretable machine learning for science with {P}y{SR} and {S}ymbolic{R}egression.jl},
  author={Cranmer, Miles},
  journal={arXiv preprint arXiv:2305.01582},
  year={2023},
  url = {https://doi.org/10.48550/arXiv.2305.01582}
}

@article{LBFGS,
  title={On the limited memory {BFGS} method for large scale optimization},
  author={Liu, Dong C and Nocedal, Jorge},
  journal={Mathematical programming},
  volume={45},
  number={1},
  pages={503--528},
  year={1989},
  publisher={Springer},
  url = {https://doi.org/10.1007/BF01589116}
}

@INPROCEEDINGS{hsin2024symbolicregressionsparsenoisy,
  author={Hsin, Junette and Agarwal, Shubhankar and Thorpe, Adam and Sentis, Luis and Fridovich-Keil, David},
  booktitle={2025 American Control Conference (ACC)}, 
  title={Symbolic Regression on Sparse and Noisy Data with {G}aussian Processes}, 
  year={2025},
  volume={},
  number={},
  doi={10.23919/ACC63710.2025.11107978}}

@inproceedings{SSINs,
 author = {DiPietro, Daniel and Xiong, Shiying and Zhu, Bo},
 booktitle = {Advances in Neural Information Processing Systems},
 title = {Sparse Symplectically Integrated Neural Networks},
 url = {https://proceedings.neurips.cc/paper_files/paper/2020/file/439fca360bc99c315c5882c4432ae7a4-Paper.pdf},
 volume = {33},
 year = {2020}
}

@article{OUprocess,
  title = {On the Theory of the {B}rownian Motion},
  author = {Uhlenbeck, G. E. and Ornstein, L. S.},
  journal = {Physical Review},
  volume = {36},
  issue = {5},
  pages = {823--841},
  numpages = {0},
  year = {1930},
  month = {Sep},
  publisher = {American Physical Society},
  doi = {10.1103/PhysRev.36.823},
}

@inproceedings{
xu2021how,
title={How Neural Networks Extrapolate: From Feedforward to Graph Neural Networks},
author={Keyulu Xu and Mozhi Zhang and Jingling Li and Simon Shaolei Du and Ken-Ichi Kawarabayashi and Stefanie Jegelka},
booktitle={International Conference on Learning Representations},
year={2021},
url={https://openreview.net/forum?id=UH-cmocLJC}
}

\end{document}